\documentclass[leqno,11pt]{amsart}
\pdfoutput=1

\usepackage{soul}

\usepackage{tikz}
\usepackage{pgfplots}
\usepackage{amsmath,amsthm,amssymb,mathrsfs} 
\usepackage{ae,aecompl}
\usepackage{times}
\usepackage{microtype}
\usepackage{dsfont}
\usepackage{amsfonts}

\usepackage{amscd}
\usepackage{txfonts}
\usepackage[margin=1in]{geometry} 
\usepackage[onehalfspacing]{setspace}
\usepackage{array}
\usepackage{enumerate}
\usepackage{kpfonts}
\usepackage[english]{babel}
\usepackage{hyperref}
\usepackage{comment}
\usepackage{tikz-cd}
\numberwithin{equation}{section}

\newcommand{\bR}{{\bf R}}

\newcommand{\cA}{\mathcal{A}}

\newcommand{\cF}{\mathcal{F}}

\newcommand{\cL}{\mathcal{L}}

\newcommand{\R}{\bR}

\newcommand{\bbR}{\mathbb{R}}

\newcommand{\bbN}{\mathbb{N}}
\newcommand{\bbP}{\mathbb{P}}

\renewcommand{\det}{\mathop\mathrm{det}\nolimits}

\renewcommand{\epsilon}{\varepsilon}

\newcommand{\e}{\varepsilon}

\def\<{\mathopen{}\left<}
\def\>{\right>\mathclose{}}
\def\({\mathopen{}\left(}
\def\){\right)\mathclose{}}

\makeatletter
\newcommand{\subalign}[1]{%
	\vcenter{%
		\Let@ \restore@math@cr \default@tag
		\baselineskip\fontdimen10 \scriptfont\tw@
		\advance\baselineskip\fontdimen12 \scriptfont\tw@
		\lineskip\thr@@\fontdimen8 \scriptfont\thr@@
		\lineskiplimit\lineskip
		\ialign{\hfil$\m@th\scriptstyle##$&$\m@th\scriptstyle{}##$\hfil\crcr
			#1\crcr
		}%
	}%
}
\makeatother

\newtheorem{theorem}{Theorem}

\newtheorem{corollary}{Corollary}

\newtheorem{example}{Example}

\newtheorem{lemma}{Lemma}

\newtheorem{proposition}{Proposition}
\newtheorem{remark}{Remark}

\numberwithin{equation}{section}

\newtheorem{innercustomthm}{Theorem}
\newenvironment{customthm}[1]
{\renewcommand\theinnercustomthm{#1}\innercustomthm}
{\endinnercustomthm}

\title[Wide neural networks: non-gaussian random fields at initialization and the NTK]{Wide neural networks: From non-gaussian random fields at initialization to the NTK geometry of training}

\author{Luis Carvalho}
\address{}
\email{}

\author{Jo\~ao Lopes Costa}
\address{}
\email{}

\author{Jos\'e Mourão}
\address{}
\email{}

\author{Gon\c{c}alo Oliveira}
\address{}
\email{galato97@gmail.com}

\begin{document}

	\begin{abstract}
		
		Recent developments in applications of artificial neural networks with over $n=10^{14}$  parameters make it extremely important to study the large $n$ behaviour of such networks. 
		
		Most works studying wide neural networks have focused on the infinite width $n \to +\infty$ limit of such networks and have shown that, at initialization, they correspond to Gaussian processes \cite{Neal,Lee}. In this work we will study their behavior for large, but finite $n$. Our main contributions are the following:
		
		\begin{itemize}
			\item The computation of the corrections to Gaussianity in terms of an asymptotic series in $n^{-\frac{1}{2}}$. The coefficients in this expansion are determined by the statistics of parameter initialization and by the activation function.
			
			\item Controlling the evolution of the outputs of finite width $n$ networks, during training, by computing deviations from the limiting infinite width case (in which the network evolves through a linear flow). This improves previous estimates \cite{Arora,Jacot,Du,Huang} and along the way we also obtain sharper decay rates for the (finite width) NTK in terms of $n$, valid during the entire training procedure. As a corollary, we also prove that, with arbitrarily high probability, the training of sufficiently wide neural networks converges to a global minimum of the corresponding quadratic loss function.   
			
			\item Estimating how the deviations from Gaussianity evolve with training in terms of $n$. In particular, using a certain metric in the space of measures we find that, along training, the resulting measure is within $n^{-\frac{1}{2}}(\log n)^{1+}$ of the time dependent Gaussian process corresponding to the infinite width network (which is explicitly given by precomposing the initial Gaussian process with the linear flow corresponding to training in the infinite width limit).
		\end{itemize}

	\end{abstract}

	\maketitle
	
	\tableofcontents

	\section{Introduction}

	\subsection{Context}
	
	Recent developments in applications of artificial neural networks with over $n=10^{14}$ parameters make it extremely important to study the large $n$ behaviour of such networks.

	Most works have focused on the infinite width $n \to +\infty$ limit of such networks and have shown that, at initialization  the limit corresponds to Gaussian random fields (GRF) \cite{Neal,Lee} while their training corresponds, for a quadratic loss,  to linear evolution with respect to the  neural tangent kernel  (NTK$^\infty$) \cite{Jacot, Arora}.

	\subsection{Summary}
	
	Our goal in this work will consist in studying the behavior of artificial neural networks with one single hidden layer for large, but finite $n$, at initialization and training.
	
	\subsubsection{Initialization}
	
	An artificial neural network, with fixed parameters (weights and biases, $\theta \in \mathbb{R}^{P_n}$), is a function from $\mathbb{R}^{n_{in}}$ to $\mathbb{R}^{n_{out}}$ (here, for the sake of simplicity, we will focus on the case $n_{in}=n_{out}=1$, even though, most of our results generalize naturally, but at the cost of overburdening the notation and obscuring some proofs). By varying the parameters we get a map $\mathcal{F}_n$ from the hyper-parameter space $\mathbb{R}^{P_n}$ to a space of functions. 
	At initialization, the parameters involved in making such a network are commonly generated at random from some predetermined distribution, which induces, by pushforward under $\mathcal{F}_n$, a sequence of distributions, $\mathbb{P}^{(n)}$, in the space of functions or a sequence of random fields.

	It is well known that in the $n \to +\infty$ limit this sequence converges to a GRF
	with many forms 
	of such a result
	being established in the literature, see for instance \cite{Neal,Lee} and references therein. While it is  important to understand such a limit, realistic networks have a finite number of neurons $n$, and one must therefore compute the deviation of the (finite dimensional) distributions $\mathbb{P}^{(n)}$ from the limiting Gaussian distribution. Some recent works \cite{Roberts} have proposed to tackle such a problem with techniques from effective quantum field theory. Instead, in the present work we take a simpler approach using the Edgeworth series technique in order to explicitly compute the corrections to Gaussianity in terms of an asymptotic series in $n^{-\frac{1}{2}}$, with coefficients determined by the statistics of parameter initialization. 

	\subsubsection{Training}
	
	The evolution of such networks during training,  can be modelled using the  finite $n$ neural tangent kernel (NTK$^{(n)}$), which is known to converge almost surely to a constant as $n \to +\infty$. Consequently, the training of such networks in this limit is equivalent to a linear flow (in the space of functions). In our work, we improve on previous estimations \cite{Jacot, Arora, Du, Huang} and compute sharper, uniform in time, decay rates for the NTK$^{(n)}$ in terms of $n$. This then results in deviations from the linear flow that models the infinite $n$ limit, which will allow us to control, both in time $t\geq0$ and in width $n$, the evolution of outputs during training. In fact, as typical of non-linear problems, we need to control ``everything at the some time'', since a careful control of NTK$^{(n)}$ requires a detailed control of the network's outputs and vice-versa. This is achieved via a continuity/bootstrap argument -- a standard technique in non-linear analysis that, to the best of our knowledge, hasn't seen widespread use in the study of artificial neural networks. Finally, a particular consequence of our results is that, under appropriate conditions, the training of sufficiently wide neural networks via gradient descent converges, with arbitrarily high probability on the initialization of parameters, to a global minimum of the corresponding quadratic loss function.    
	
	These results them allow us to  estimate, in terms of $n$, how the deviations from Gaussianity, that we have derived at initialization, evolve during training. We start by showing that the probability measure genertaed by the infinite width network is a time dependent Gaussian process which is explicitly given by precomposing the initial Gaussian process with the linear flow corresponding to training in the infinite width limit. Then,
	using the Prokhorof metric in the space of measures, we find that, along the training procedure, the measure associated with the finite width network is within $n^{-\frac{1}{2}}(\log n)^+$ of the time dependent Gaussian process corresponding to the infinite width network.

	\subsection{Main results}
	
	Before stating the main results of this article we require some preparation. 
	We shall refer  to the weights and biases of such a network as parameters and denote them by
	$$ \theta:= ( (W^{(2)}_1 , \ldots , W^{(2)}_n,W^{(1)}_1, \ldots , W^{(1)}_n), (b^{(2)}_{1}, \ldots b^{(2)}_n,b^{(1)}_1 , \ldots , b^{(1)}_n) ) \in (\mathbb{R}^n \times \mathbb{R}^n) \times (\mathbb{R}^n \times \mathbb{R}^n) \cong \mathbb{R}^{4n} =: \mathbb{R}^{P_n}.$$
	These, together with a nonlinear activation function $\sigma : \mathbb{R} \to \mathbb{R}$, determine the scalar valued function encoding the network
	\begin{equation}\label{eq:f_theta 1 hidden layer INTRO}
		f^{(n)}_\theta (x)  = \frac{1}{\sqrt{n}} \sum_{j=1}^{n} \left( W^{(2)}_{j} \ \sigma \left(  W^{(1)}_{j} x +b^{(1)}_j \right) + b^{(2)}_j \right) .
	\end{equation}	
	As a consequence, we obtain a map $\cF^{(n)}$, which to a set of parameters $\theta$ associates  the function $f^{(n)}_\theta \in \mathrm{Map}([a,b], \mathbb{R})$.

	\subsubsection{Initialization}
	
	Typically in applications, the parameters $\theta$ are initialized at random from a pre-determined probability measure on $\mathbb{R}^{4n}$. For each $n$, this induces (by pushforward via $\mathcal{F}_n$) probability measures
	$\mathbb{P}^{(n)}$ on the infinite dimensional space of functions ${\rm Map}([a,b], \mathbb{R})$. Under some very natural assumptions reviewed in \ref{ss:convergence to Gaussian}, and which we also recall below in this introduction, it follows from the central limit theorem that $\mathbb{P}^{(n)}$
	converges to a Gaussian measure
	on ${\rm Map}(\mathbb{R}, \mathbb{R})$, called a GRF. The main goal of the first part of this work is to compute the deviation of $\mathbb{P}^{(n)}$ from Gaussianity. In order to investigate such deviations it is convenient to work instead with the finite dimensional distributions obtained by evaluating the functions at a finite set of points. Notice that, as we recall in \ref{ss:measures and Kolmogorov}, it follows from Kolmogorov's extension theorem that there is no loss of generality in doing so.
	
	We now consider the finite dimensional distributions obtained  by fixing $k \in \mathbb{N}$ and $x_1, \ldots , x_k \in[a,b]$,  considering the map,  which to a set of parameters associates the resulting network evaluated at these points
	\begin{eqnarray*}
		\mathbb{R}^{4n} & \longrightarrow & \mathrm{Map}([a,b],\mathbb{R}) \to \mathbb{R}^{k} , \\
		\theta & \longmapsto & \left(f_\theta(x_1), \cdots , f_\theta(x_k)\right) \,,  
	\end{eqnarray*}
	and taking the pushforward of the initially fixed distribution on $\mathbb{R}^{4n}$ to one in $\mathbb{R}^k$ which we denote by $Q^{(n)}_{x_1, \ldots , x_k}$. 
	
	Working with such finite dimensional distributions it is easy to understand the convergence to a Gaussian process as follows. From now on, suppose that all the tuples $(W^{(2)}_{j},W^{(1)}_{j}, b^{(2)}_{j}, b^{(1)}_{j}) \sim (W^{(2)}, W^{(1)}, b^{(2)}, b^{(1)})$ are i.i.d. and that
	both $W^{(2)}$ and $b^{(2)}$ have zero mean. Then the central limit theorem applies to \ref{eq:f_theta 1 hidden layer INTRO} and gives that 
	\begin{equation}
		\bbP^{(n)}_{x_1, \ldots , x_k} \rightarrow \bbP^C_{x_1, \ldots , x_k},
	\end{equation}
	with $\bbP^C_{x_1, \ldots , x_k}$  a Gaussian distribution on $\mathbb{R}^k$, with covariance $C(x_i,x_j)=\mathbb{E}[p(x_i)p(x_j)]$, where $p(x)=W^{(2)} \sigma (W^{(1)}x+b^{(1)})+b^{(2)}$ and $\mathbb{E}[\cdot]$ denotes the expectation with respect to the initially fixed distribution on the parameters $(W^{(2)}, W^{(1)}, b^{(2)}, b^{(1)})$. More explicitly
	\begin{eqnarray}\label{eq:Gaussian Limit INTRO}
		d\bbP^C_{x_1, \ldots , x_k}(y_1 , \ldots , y_k) = \frac{e^{- \frac{1}{2}  C^{-1}_{ab} y_ay_b}}{(2\pi \det(C_{ij}))^{k/2}} dy^1 \ldots dy^k,
	\end{eqnarray}
	where $C_{ij} := C(x_i,x_j)$ and we used Einstein convention of summing over repeated indices. In particular, if $k=1$ then $C=\mathbb{E}[p(x)^2]$ is the second moment of the random variable $p(x)$ which we will denote by $\mu_2$. Then, in this $k=1$ situation, we have 
	\begin{eqnarray}\label{eq:Gaussian Limit INTRO 2}
		d\bbP^C_{x}(y) = \frac{e^{- \frac{y^2}{2\mu_2} }}{\sqrt{2\pi \mu_2}} dy.
	\end{eqnarray}
	By combining the central limit theorem above with the Kolmogorov extension theorem we
	conclude that the finite dimensional distributions (\ref{eq:Gaussian Limit INTRO})
	define a Gaussian measure on ${\rm Map}([a,b], \bbR)$ with covariance ${\bf C}$.
	Furthermore, if the nonlinear activation function $\sigma$ is Lipschitz, we shall prove in Theorem \ref{app-e T5} that the Gaussian measure
	$\bbP^{C}$ is in fact supported on the smaller and better behaved space of continuous functions, $\mathcal{C} = C([a.b], \bbR)$.

	Our first main result in this article is Theorem \ref{thm: Asymptotic} (see Figure~\ref{figMu3_0} and Figure~\ref{figMu3_neq0}) which, for finite $n$, explicitly computes the deviations from Gaussianity in the distribution at initialization $\bbP^{(n)}_{x_1, \ldots , x_k}$. Here, in the introduction, we shall simply state a special version of this result for $k=1$ which should be compared to equation \ref{eq:Gaussian Limit INTRO 2} and regarded as a refinement of that formula which is valid for large, but finite, $n$.
	
	\begin{customthm}{A}[Special case of Theorem \ref{thm: Asymptotic}]
		\label{thm: Asymptotic}
		Let $x \in \mathbb{R}$ and the tuples $(W^{(2)}_i,b^{(2)}_i,W^{(1)}_i,b^{(1)}_i)$ be independent and identically distributed, for all $i \in \lbrace 1, \ldots , n \rbrace$, and that $W^{(2)}_i,b^{(2)}_i$ have zero average. Then,
		$$d\bbP^{(n)}_{x}(y)=\nu_{x}^{(n)}(y) dy$$ 
		with the density $\nu_{x}^{(n)}$ satisfying
		\begin{align*}
			\nu_{x}^{(n)}(y) & =\frac{e^{-\frac{y^2}{2\mu_2}}}{\sqrt{2\pi\mu_2}} + \frac{1}{\sqrt{n}}  \frac{\mu_{3}}{12 \sqrt{2} \mu_2^{3/2}} H_3 \left( \frac{y}{\sqrt{2\mu_2}} \right) \frac{e^{-\frac{y^2}{2\mu_2}}}{\sqrt{2\pi\mu_2}} \\
			& \ \ \ \  +  \frac{1}{n} \left( \frac{\mu_4-3\mu_2^2}{96\mu_2^2} H_4 \left( \frac{y}{\sqrt{2\mu_2}} \right) + \frac{\mu_{3}^2}{576 \mu_2^3} H_6 \left( \frac{y}{\sqrt{2\mu_2}} \right) \right) \frac{e^{-\frac{y^2}{2\mu_2}}}{\sqrt{2\pi\mu_2}} + O(n^{-3/2})  ,
		\end{align*}
		where $H_n(\cdot)$ denotes the $n$-th Hermite polynomial and $\mu_r= \mathbb{E}[p(x)^r]$.
	\end{customthm}
	
	\begin{remark}
		Again, we remark that this result is an illustration of our main result on $\bbP^{(n)}_{x_1, \ldots , x_k}$, presented in more detail in \ref{ss:deviations from Gaussian} and proven in \ref{ap:proof of theorem} is the following:
	\end{remark}

	\begin{figure}[t]
		\label{figMu3_0}
		\includegraphics[width=8cm]{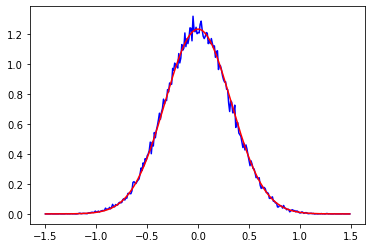}
		\caption{
			An experimental realization of Theorem A revealing Gaussian behavior (in particular $\mu_3=0$).
			The blue curve corresponds to the density function obtained by sampling, 80.000 times, the evaluation, at $x=1$, of a neural network with 2.000 neurons/perceptrons, sigmoide non-linearity and  parameters sampled uniformly in the interval $[-1/2,1/2]$. The red curve is the prediction of the theorem, excluding the $O(n^{-3/2})$ terms.}
		\centering
	\end{figure}

	\begin{figure}[t]
		\label{figMu3_neq0}
		\includegraphics[width=8cm]{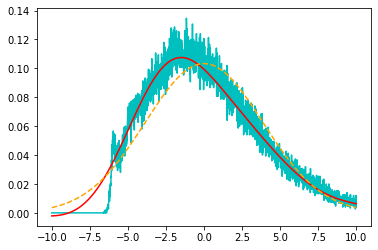}
		\caption{ 
			An experimental realization of Theorem A revealing ``strong'' deviations from Gaussianity (in particular $\mu_3\neq0$).
			The cyan curve corresponds to the density function obtained by sampling, 150.000 times, the evaluation, at $x=1$, of a neural network with 300 neurons/perceptrons, sigmoide non-linearity and  parameters sampled from a distribution with momenta $m_2\approx11,6$, $m_3\approx607,6$ and $m_4\approx34.131,4$. The red curve is the prediction of the theorem, excluding the $O(n^{-3/2})$ terms, and the orange dashed curve is the Gaussian with the same $\mu_2$, presented for comparison.}
		\centering
	\end{figure}

	\subsubsection{Training}
	Now consider that we are given a training set $\{X_i,Y_i\}_{i=1}^N$ and that we want to use it to train our networks~\eqref{eq:f_theta 1 hidden layer INTRO} by minimizing a quadratic loss function via gradient descent. With that in mind we initialize our parameters $\theta_{0}=(W^{(2)}_{0,j},b^{(2)}_{0,j},W^{(1)}_{0,j},b^{(1)}_{0,j})_{j\in\mathbb{N}}$ as before (but now we will also assume that they have finite third order momenta). We then have a well defined dynamical system in parameter space $t\mapsto \theta(t)$ and a corresponding well posed evolution    
	for the outputs 
	$$y(t)=f^{(n)}_{\theta(t)}(x)\;,$$
	of a finite $n$ width network (recall~\eqref{eq:f_theta 1 hidden layer INTRO} once again).  
	
	In this paper we also establish (see Appendix~\ref{appEvolution} for the corresponding proofs) various estimates that provide, under appropriate conditions detailed in  section~\ref{secEvolutionOutputs}, a clear description of the complete dynamics of the outputs of sufficiently wide, but finite, networks. In particular we will show that (see Theorem~\ref{thmEvolutionComplete} for a complete version of this result)

	\begin{customthm}{B}[Simplified version of Theorem \ref{thmEvolutionComplete}]
		By increasing the width $n$, then with arbitrarily high probability on (parameter) initialization: 
		\begin{itemize}
			\item The {\em difference between the finite width NTK and its infinite limit} satisfies an estimate of the form
			\begin{align}
				|\mathrm{NTK}_{zw}^{(n)}(\theta(t))-\mathrm{NTK}_{zw}^\infty|
				&\leq C
				\frac{\;(\log n)^{1/2+}}{\sqrt{n}} ,
			\end{align}
			for all $t\geq 0$, all $|z|\;,|w|\leq 1$.
			\item The {\em  training outputs (of the finite width neural network) convergence exponentially to their labels:}
			\begin{equation}
				\label{decayError0}
				|f^{(n)}_{\theta(t)}(X_l)-Y_l|\leq C e^{-(\lambda_{\infty}-\epsilon )t}\;, 
			\end{equation}
			where $\lambda_{\infty}>0$ is the minimum eigenvalue of $\mathrm{NTK}_{X_iX_j}^\infty$~\cite{NTK>0}.
			\item The {\em training is stable}, in the sense that the output dynamics remains close to the linear (infinite width) dynamics provided by $\mathrm{NTK}^\infty$. More precisely we show that, for all $|x|\leq 1$, 
			\begin{equation}
				|f^{(n)}_{\theta(t)}(x)-f^{\infty}_{t}(x)|\leq
				C \frac{(\log(n))^{1/2+}}{\sqrt{n}}\;.
			\end{equation} 
		\end{itemize}
	\end{customthm}

	Finally, fix $(x_1, \ldots , x_k) \in \mathbb{R}^k$ and consider the evolution, during training, of the measures in the space of outputs. 
	At each training time $t\geq 0$ this gives two measures $\mathbb{P}^{(n)}_{x_1, \ldots , x_k}(t)$ and $\mathbb{P}^{\infty}_{x_1, \ldots , x_k}(t)$, respectively obtained from training  using a finite width $n$ network and using an infinite network (note that, according to Theorem~\ref{prop: density for infinite with along training}, $\mathbb{P}^{\infty}(t)$  is Gaussian, for all $t\geq0$).  Our goal is to estimate how these two differ with $n$ as we evolve in $t$. To this purpose we shall measure the distance between these two metrics using the Prokhorov metric. For two measures $\mu_1, \mu_2$, this is denoted by $\pi(\mu_1, \mu_2)$ and defined as the infimum over the set of positive $\rho>0$ that simultaneously satisfies $\mu_1(A) \leq \mu_2(A^\rho)+\rho$ and $\mu_2(A) \leq \mu_1(A^\rho)+\rho$, with $A^\rho=\bigcup_{a \in A} B_\rho(a)$. Then, we prove the following.
	
	\begin{customthm}{C}\label{thm:distance between measures INTRO}
		Let $(x_1, \ldots , x_k) \in \mathbb{R}^k$. Then, there is a constant $C>0$, independent of $n$, such that for all $t \geq 0$
		$$\pi \Big(  \mathbb{P}^{(n)}_{x_1, \ldots , x_k}(t) , \mathbb{P}^{\infty}_{x_1, \ldots , x_k}(t)  \Big) \leq C \frac{(\log(n))^{1+}}{\sqrt{n}} .$$
	\end{customthm}

	\subsection{Related works}
	
	There as been extensive work concerning wide neural networks and their asymptotics in infinite width limit. Here we highlight some of those results that, to the best of our knowledge, are more closely related to the developments presented in this paper. 
	
	That the $n \to +\infty$ width limit  converges  to a Gaussian process (in an appropriate sense clarified earlier) is by now a classical result in machine learning, first established in~\cite{Neal} for the shallow network case and generalized recently~\cite{Lee} to the deep network case. See also \cite{Lee2,Ghorbani} for a performance comparison with Kernel methods.
	
	As already mentioned, other recent works have already studied deviations from Gaussianity that occur in the large but finite width case. For instance,~\cite{Roberts} deals with deep networks and relies on techniques from effective quantum field theory. Also, the kind of parameter initialization they privilege (standard Gaussian initialization with zero mean) leads to deviations as powers of $n^{-1}$, where $n$ determines the width size. In contrast, in the present work we take a simpler approach using the Edgeworth series technique in order to explicitly compute the corrections to Gaussianity in terms of an asymptotic series in $n^{-\frac{1}{2}}$, with coefficients determined by the statistics of parameter initialization.
	
	The paper~\cite{Arora} also concerns deep neural networks. In that setting it is shown, in
	particular, that  at initialization the difference between the infinite limit NTK$^{\infty}$ and the finite width 
	NTK$^{(n)}$ can be made arbitrarily small by increasing the width of all hidden layers. They also show that  after training the difference between the outputs of the finite and the corresponding infinite width networks can be made arbitrarily small by increasing the width. When compared to the results presented here the clear advantage of~\cite{Arora} is that it deals with deep networks, while we only consider single-layered networks. 
	Nonetheless, their bounds do not have an explicit quantitative dependence on the width of the layers and their results only apply to ReLu non-linearities; a situation that is largely fixed in our paper. 
	
	The results in~\cite{Huang} provide various estimates, in terms of the width $n$, controlling the behavior of the NTK, and higher order generalization of this object, and of the outputs during training. Once again their results apply to deep neural networks, but unfortunately the estimates they provide depend on increasing functions of training time $t$, a shortfall that, in particular, does not allow to control the $t\rightarrow\infty$ limit; nonetheless the estimates provided  in~\cite{Huang} are sufficient to show that outputs of training data can be made arbitrarily close to their labels, by appropriately increasing the width $n$.   
	
	In the, already mentioned, extensive work carried out in~\cite{Roberts} one can also find results concerning the behavior of outputs of wide neural networks, where it is shown that after training the difference between outputs and labels can be controlled up to an error of the order of $n^{-2}$, where $n$ determines the width of the networks.

	Concerning the behavior of outputs during training the  results in~\cite{Du,Du2} are, to the best of our knowledge, the ones with biggest overlap with the ones presented here. They also consider shallow networks and they also obtain uniform, in training  time, estimates that control the output dynamics in terms of a width parameter. In our opinion the results presented in our paper have the following advantages when compared to~\cite{Du}: i) the results in~\cite{Du} are restricted to the ReLu non-linearity and ii) require a initialization of the parameters of the output layer to be done with a uniform distribution; iii) these two restrictions are largely lifted in the current paper; this requires overcoming some new challenges since we can no longer rely on the simplified explicit computations allowed by the ReLu non-linearity and we have to deal with outputs that aren't necessarily compactly supported, something that in~\cite{Du} was assured by the the referred restrictive initialization of output parameter using a uniform distribution. Finally, we also provide a constructive proof based on the continuity/bootstrap method. Other relevant results in this context can be found in~\cite{Chizat}.   
	
	Other interesting and related works which investigate wide neural networks and which may be amenable to the techniques developed here are for instance \cite{Cao}.
	
	An interesting and important direction that can be pursued using the techniques we develop in this article is to further investigate the consequences of overparametrization 
	along the lines of \cite{Bartlett,Mei,Neyshabur,Oymak,Xu,Yang,Arora2}.
	
	The paper \cite{chen} by Chen et al obtains mean field theoretic results about fluctuations of the measure on the space of empirical measures, of order $m^{-1/2}$ at initialization and during training, where $m$ is the number of neurons of a single layered neural net. It would be very interesting to better understand the relation of this mean field theoretic approach for probability measures on the space of measures on $\R^4$ with our more random field theoretic approach
	using measures on the space of functions, $C([a,b], \R)$. 
	
	Along somewhat similar lines, it may also be interesting to investigate how our methods can be adapted to other parametrizations such as in \cite{Yang2}, and to better investigate feature learning.
	
	It would also be an interesting project for future work to try and understand how to use our techniques to better understand Physics informed neural networks and some interesting phenomena which occur there. See for instance \cite{Karniadakis,Wang}.

	\section{Deviations from gaussianity at initialization}
	\label{s-2}

	\subsection{Wide neural networks to be considered} \label{s-2.1}
	We consider a neural network with one hidden layer and width $n \in \mathbb{N}$ and input and output taking values in the real numbers. For convenience we will further assume that the inputs are in a fixed nonempty compact interval  $[a,b] \subset \bbR$. 
	This means that for each choice of parameters
	$\theta \in \mathbb{R}^{P_n}$, with $P_n=4n$, which is short for
	$$ \theta:= ( (W^{(2)}_1 , \ldots , W^{(2)}_n,W^{(1)}_1, \ldots , W^{(1)}_n), (b^{(2)}_{1}, \ldots b^{(2)}_n,b^{(1)}_1 , \ldots , b^{(1)}_n) ) \in (\mathbb{R}^n \times \mathbb{R}^n) \times (\mathbb{R}^n \times \mathbb{R}^n) ,$$
	consisting of weights and biases respectively,
	we construct a  function
	$$f^{(n)}_\theta : [a,b] \, \to \, \mathbb{R},$$
	given by
	\begin{equation}\label{eq:f_theta 1 hidden layer 0}
		f^{(n)}_\theta (x)  = \frac{1}{\sqrt{n}} \sum_{j=1}^{n} \left( W^{(2)}_{j} \ \sigma \left(  W^{(1)}_{j} x +b^{(1)}_j \right) + b^{(2)}_j \right) ,
	\end{equation}	
	where $\sigma : \mathbb{R}  \to \mathbb{R}$ is a nonlinear continuous scalar function called the activation function. 
	Therefore, for every $n \in \bbN$,  the neural network is  the map,
	\begin{eqnarray}
		\label{22a}
		\mathcal{F}^{(n)} : \,  \mathbb{R}^{ 4n} 
		& \longrightarrow & C([a,b], \bbR) \subset {\rm Map}([a,b], \mathbb{R}) \\
		\nonumber \theta & \longmapsto & f^{(n)}_\theta \, . 
	\end{eqnarray}

	\subsection{The induced probability measures on function space at initialization}\label{ss:measures and Kolmogorov}
	
	At initialization the parameters $\theta$ are randomly selected from a probability measure on $\mathbb{R}^{4n}$, which we denote  by $Q^{(n)}$ with density 
	$\nu_n(\theta)$, such that
	$dQ^{(n)} = \nu_n(\theta) d \theta$ where $d\theta:=d\theta^1 \ldots d\theta^{4n}$. By taking the pushforward under $\mathcal{F}^{(n)}$ we  obtain a sequence of probability measures
	$\mathbb{P}^{(n)}$ on 
	the infinite dimensional space of continuous functions 
	on $[a,b]$,
	\begin{equation}
		\label{ee-243}
		{\mathcal C} := C([a,b], \mathbb{R}) \, , \qquad
		\mathbb{P}^{(n)} := \left(\mathcal{F}^{(n)}\right)_*(Q^{(n)}) \, ,
	\end{equation}
	supported on the cones,
	\begin{equation}
		\label{ee-244}
		\mathcal{C}^{(n)} := \mathcal{F}^{(n)}(\mathbb{R}^{4n}) 
		\subset \, \mathcal{C} \subset {\rm Map}([a,b], \mathbb{R}) \, ,     
	\end{equation}
	of maximal dimension $3n+1$.
	Under natural assumptions, 
	that we will review below, 
	the central limit theorem implies that the sequence $\mathbb{P}^{(n)}$
	converges to a Gaussian measure
	on $\mathcal{C}=C([a,b], \mathbb{R})$,
	\begin{equation}
		\label{ee-245}
		\mathbb{P}_C^{(\infty)} =
		\lim_{n \to \infty} \, \mathbb{P}^{(n)} \, ,
	\end{equation}
	where $C$ denotes the covariance. 
	Informally the distributions $\mathbb{P}^{(n)}$ define a sequence of (in general) non-Gaussian random fields converging to a Gaussian random field. 
	
	One main goal of this work will 
	be to study the deviation of $\mathbb{P}^{(n)}$  from Gaussianity. 
	From Kolmogorov extension theorem we know that probability measures
	on ${\rm Map}([a,b], \mathbb{R})$
	define and are defined by all possible 
	finite dimensional distributions obtained by evaluating the fields at a finite number of points. More precisely,
	for $k \in \mathbb{N}$ and
	$x_1, \ldots , x_k \in [a,b]$  denote by $p_{x_1, \ldots , x_k}$  the evaluation map
	\begin{eqnarray*}
		p_{x_1, \ldots , x_k} \, : \, 
		{\rm Map}([a,b], \mathbb{R})
		& \longrightarrow & \mathbb{R}^k
		\\
		f  & \longmapsto &  
		\left(f(x_1) \ldots , f(x_k)\right)    \, . 
	\end{eqnarray*}
	We can therefore consider the composition
	\begin{eqnarray}\label{eq:Fnx}
		{\mathcal F}^{(n)}_{x_1, \ldots , x_k} = 
		p_{x_1, \ldots , x_k}  \circ {\mathcal F}^{(n)} \, : \, \mathbb{R}^{4n} & \longrightarrow & \mathrm{Map}([a,b] , \mathbb{R} ) \to \mathbb{R}^{k} , \\ \nonumber
		\theta & \longmapsto & \left(f_\theta(x_1), \cdots , f_\theta(x_k)\right) \,,  
	\end{eqnarray}
	and the corresponding pushforward of the 
	measure $Q^{(n)}$,
	\begin{equation}\label{eq:Finite dimensional pushforward measures}
		\bbP^{(n)}_{x_1, \ldots , x_k} := \left({\mathcal F}^{(n)}_{x_1, \ldots , x_k}\right)_* (Q^{(n)})\,.
	\end{equation}
	By construction, the measures $\bbP^{(n)}_{x_1, \ldots , x_k}$ satisfy the following two compatibility conditions
	\begin{equation}
		\label{ee-cc1}
		\int_{I_1 \times \ldots \times I_k} d \bbP^{(n)}_{x_{\sigma(1)}, \ldots , x_{\sigma(k)}} = \int_{I_{\sigma^{-1}(1)} \times \ldots \times I_{\sigma^{-1}(k)}} d \bbP^{(n)}_{x_1, \ldots , x_k},
	\end{equation}
	for $I_i \subset [a,b]$ and $\sigma$ a permutation of $\lbrace 1, \ldots , k \rbrace$, and 
	\begin{equation}
		\label{ee-cc2}
		\int_{I_1 \times \ldots \times I_k \times \mathbb{R}\times \ldots \times \mathbb{R} } d \bbP^{(n)}_{x_{1}, \ldots , x_{k}, x_{k+1} , \ldots , x_{k+m} } = \int_{I_1 \times \ldots \times I_k} d \bbP^{(n)}_{x_{1}, \ldots , x_{k}}, 
	\end{equation}
	for $m \in \mathbb{N}$. 
	Reciprocally,  from Kolmogorov's extension theorem (Theorem 9.2, p. 37 of \cite{Ya}), {\it any} collection of  finite dimensional distributions,  
	\begin{equation}
		\label{ee-2.9}
		\left\{\bbP^{(n)}_{x_1, \ldots , x_k}
		\right\}_{k \in \mathbb{N}, x_1, \dots, x_k \in [a,b]} \, , 
	\end{equation}
	satisfying the consistency conditions (\ref{ee-cc1}), (\ref{ee-cc2})
	defines a unique measure $\mathbb{P}^{(n)}$ on ${\rm Map}([a,b], \mathbb{R})$.

	\subsection{Convergence to the Gaussian Random Field}\label{ss:convergence to Gaussian}

	Recall that by definition a measure $\mathbb{P}$ is Gaussian on the space of functions 
	${\rm Map}([a,b], \mathbb{R})$ (or on $\mathcal{C}$)
	if  all the measures 
	$\bbP_{x_1, \ldots , x_k}$ on $\mathbb{R}^{k}$, for all $k \in \mathbb{N}, \, x_1, \ldots, x_k \in [a,b]$,  are Gaussian.

	We shall now consider the standard situation in which all perceptrons in the hidden layer are independent and their parameters identically distributed. In this case, as was already discussed, as the number of artificial neurons $n \rightarrow\infty$, the measures $\mathbb{P}^{(n)}$ converge to a Gaussian measure on $\mathrm{Map}([a,b], \mathbb{R})$.
	
	In fact, this can easily be seen by considering the expression (\ref{eq:f_theta 1 hidden layer 0}) for 
	the function $f^{(n)}_\theta$. This can be written as 
	\begin{equation}
		\label{eq:f_theta_2}
		f^{(n)}_\theta(x)= \frac{1}{\sqrt{n}} \sum_{j=1}^n p_{\hat \theta_j}(x),
	\end{equation}
	where $\hat \theta = (W^{(2)}, W^{(1)}, b^{(2)}, b^{(1)}) \in \bbR^4$,
	$$p_{\hat \theta_j}(x)=W^{(2)}_{j} \ \sigma \left(  W^{(1)}_{j} x +b^{(1)}_j \right) + b^{(2)}_j ,$$
	denotes the $j$-th perceptron. From now on we will assume that all 
	the tuples $(W^{(2)}_{j},W^{(1)}_{j}, b^{(2)}_{j}, b^{(1)}_{j}) \sim (W^{(2)}, W^{(1)}, b^{(2)}, b^{(1)})$ are i.i.d.
	and therefore 
	the 
	perceptrons are 
	also i.i.d.
	Furthermore,
	we will assume that
	both $W^{(2)}$ and $b^{(2)}$ have zero mean. Then the central limit theorem applies to (\ref{eq:f_theta_2}) and gives
	
	\begin{equation}
		\lim_{n \to + \infty} 
		\mathbb{P}^{(n)} = \mathbb{P}^{\bf C} , 
	\end{equation}
	where $\mathbb{P}_{\bf C}$ denotes the 
	Gaussian measure on ${\rm Map}([a,b], \bbR)$ with covariance  ${\bf C}$ given by the following function on $[a,b] \times [a,b] \subset \mathbb{R}^2$,
	\begin{equation}
		\label{ee-2.12b}
		{\bf C}(x,x')= 
		\mathbb{E}_f [f(x) f(x')]
		:= \mathbb{E}_{\hat \theta} [p(x) p(x')].
	\end{equation}
	Equivalently, the finite dimensional distributions $\bbP^{\bf C}_{x_1, \dots, x_k}$
	on $\mathbb{R}^k$ 
	are Gaussian, with  covariance $k \times k$ matrix 
	given by
	$C_{ab} := {\bf C}(x_a,x_b)$ 
	\begin{eqnarray}\label{eq:Gaussian Limit}
		d\bbP^{{\bf C}}_{x_1, \ldots , x_k}(y_1 , \ldots , y_k) = \frac{e^{- \frac{1}{2}  C^{-1}_{ab} y_ay_b}}{\sqrt{ (2\pi)^k \det(C)}} dy^1 \ldots dy^k,
	\end{eqnarray}
	where we have used the Einstein summation convention over repeated indices.
	We will compute an asymptotic expansion, valid for large $n$, which refines this result. See Theorem \ref{thm: Asymptotic} for the precise statement. 
	
	\begin{remark}\hfill
		\begin{itemize}
			\item We are assuming that $C_{ab}$ is non-degenerate. This turns out to be the case if the $x_a$'s are distinct and provided that the activation function $\sigma$ is sufficiently regular and non-polynomial (see~\cite{NTK>0} for more details). 
			\item It follows from Theorem 1.1.1 in \cite{Rivera} that if $x \mapsto {\bf C}(x,x')$ is of class $C^{k+1}$, for all $x' \in [a,b]$, then the measure $\bbP^{{\bf C}}$ is supported on the set of of class $C^k$-functions from $[a,b]$ to $\bbR$. This is clearly the case if $\sigma $ is itself of class $C^{k+1}$. 
		\end{itemize}
	\end{remark}


	\subsection{Quantitative deviations from Gaussianity}\label{ss:deviations from Gaussian}
	
	In this section we state one of our main results, Theorem \ref{thm: Asymptotic}, and one of its simple consequences presented as Example~\ref{cor: Asymptotic}. For large but finite $n$, we estimate  the deviation 
	of the finite width distribution $\bbP^{(n)}$
	from the Gaussian $\mathbb{P}^{\bf C}$. The proofs of these results are presented in Appendix~\ref{appEdgeworth}.
	
	We shall state the result in terms of the random function generated by a single perceptron
	\begin{equation}\label{eq:perceptron}
		p(x):= p_{\hat \theta}(x) = W^{(2)} \sigma (W^{(1)}x+b^{(1)})+b^{(2)}\,.
	\end{equation}
	Then, for $x_1 , \ldots , x_k \in \mathbb{R}$ and $r_1, \ldots , r_k \in \mathbb{N}_0$, we shall write
	$$  \mu_{r_1, \ldots , r_k} : =
	\mu_{r_1, \ldots , r_k}(x_1, \cdots, x_k) = \mathbb{E} \left[ p(x_1)^{r_1} \ldots p(x_k)^{r_k} \right],$$
	for the moments of the random variables $(p(x_1), \ldots , p(x_k))$ and
	$$  C(x_a, x_b) = \mathbb{E}[p(x_a)p(x_b)],$$
	for its covariance matrix. We further define, for every choice of points $x_1, \dots, x_k$, a set of real numbers
	$\lambda_{r_1, \ldots , r_k}=\lambda_{r_1, \ldots , r_k}(x_1, \dots, x_k)$, known as the cumulants of $(p(x_1), \ldots , p(x_k))$. These are determined by the $\mu_{r_1, \ldots , r_k}$ by requiring the equality of the following formal power series in the variables $q_1, \ldots , q_k$
	\begin{align}\label{eq:cumulants from moments 0}
		\sum_{r_1+ \ldots + r_k=1}^\infty \lambda_{r_1, \ldots r_k} \frac{(q_1)^{r_1} \ldots (q_k)^{r_k} }{r_1! \ldots r_k!} = \sum_{j=1}^\infty \frac{(-1)^{j+1}}{j} \left( \sum_{s_1+ \ldots + s_k=1}^\infty \mu_{s_1, \ldots s_k} \frac{(q_1)^{s_1} \ldots (q_k)^{s_k} }{s_1! \ldots s_k!} \right)^j ,
	\end{align}
	
	We are now ready to state our main result, the proof of which will be presented in \ref{ap:proof of theorem}.
	
	\begin{theorem}\label{thm: Asymptotic}
		Let the tuples $\hat{\theta}^{(i)}=(W^{(2)}_i,b^{(2)}_i,W^{(1)}_i,b^{(1)}_i) \sim \hat{\theta}$ be independent and identically distributed, for all $i \in \lbrace 1, \ldots , n \rbrace$. Assume also that $W^{(2)}_i,b^{(2)}_i$ have zero average, that all moments of $p_{\hat{\theta}}(x)$ are finite for any $x \in \mathbb{R}$, and let
		\begin{align}\label{eq:f_theta}
			\mathcal{F}^{(n)}(\theta)(x) = f^{(n)}_\theta (x) & = \frac{1}{\sqrt{n}} \sum_{j=1}^{n} \left( W^{(2)}_{j} \ \sigma \left(  W^{(1)}_{j} x +b^{(1)} \right) + b^{(2)}_j \right)\,.
		\end{align}
		Then, for all $x_1 , \ldots , x_k \in [a,b]$ and $l \in \mathbb{N}$, the measures
		$$d\bbP^{(n)}_{x_1, \ldots , x_k}(y_1 , \ldots , y_k)=\nu_{x_1, \ldots , x_k}(y_1 , \ldots , y_k) dy^1 \cdots dy^k\;,$$
		on $\mathbb{R}^k$, have densities $\nu_{x_1, \ldots , x_k}(y_1 , \ldots , y_k)$ given by
		\begin{equation}\label{eq:Distribution f}
			\exp \left( \sum_{r_1+ \ldots + r_k>2}^l \frac{1}{n^{\frac{r_1 + \ldots + r_k}{2}-1}}  \frac{\lambda_{r_1, \ldots r_k}}{r_1! \ldots r_k!} (-\partial_{y_1})^{r_1}  \ldots (-\partial_{y_k})^{r_k}  \right) \frac{e^{- \frac{1}{2} C^{-1}_{ij}  y_i y_j}}{ \sqrt{(2\pi)^{k} \det(C)} } +o(n^{-l/2+1})\; .
		\end{equation}
		
	\end{theorem}

	\begin{remark}\label{rem: vanish of n_1^{-1/2}}
		It is interesting to ask if there are conditions that can be put on the distributions of $W^{(2)}_i,b^{(2)}_i$ and $W^{(1)}_i,b^{(1)}_i$, i.e. on the perceptrons, such that the convergence to the Gaussian process can be accelerated. 
		
		This will be the case if we can make the term of order $n^{-1/2}$ to vanish, which can be achieved
		by making $\mu_{r_1, \ldots , r_k}=0$ whenever $r_1 + \ldots + r_k=3$. 
		
		There are reasonable conditions on the distributions of $W^{(2)},b^{(2)}$ which guarantee the vanishing of the third moments. For example, if in addition to the average of $W^{(2)},b^{(2)}$ vanishing, also their third moment does. This is true if they are distributed symmetrically with respect to $0$ (the average), as is the case for the normal distribution. 
		
	\end{remark}

	In order to illustrate how this theorem can be used in practice we give the following immediate consequence, which is  proven in Appendix \ref{ap:proof of corollary}
	
	\begin{example}[Case $k=1$]\label{cor: Asymptotic}
		Let $x \in \mathbb{R}$. Then, for $(W^{(2)}_i,b^{(2)}_i,W^{(1)}_i,b^{(1)}_i)$ and $f_\theta(\cdot)$ as in Theorem \ref{thm: Asymptotic}, we have
		\begin{align*}
			\nu_{x}(y) & =\frac{e^{-\frac{y^2}{2\mu_2}}}{\sqrt{2\pi\mu_2}} + \frac{1}{\sqrt{n}}  \frac{\mu_{3}}{12 \sqrt{2} \mu_2^{3/2}} H_3 \left( \frac{y}{\sqrt{2\mu_2}} \right) \frac{e^{-\frac{y^2}{2\mu_2}}}{\sqrt{2\pi\mu_2}} \\
			& \ \ \ \  +  \frac{1}{n} \left( \frac{\mu_4-3\mu_2^2}{96\mu_2^2} H_4 \left( \frac{y}{\sqrt{2\mu_2}} \right) + \frac{\mu_{3}^2}{576 \mu_2^3} H_6 \left( \frac{y}{\sqrt{2\mu_2}} \right) \right) \frac{e^{-\frac{y^2}{2\mu_2}}}{\sqrt{2\pi\mu_2}} + O(n^{-3/2})  ,
		\end{align*}
		where $H_n(\cdot)$ denotes the $n$-th Hermite polynomial and $\mu_r= \mathbb{E}[p(x)^r]$.
	\end{example}

	\section{The geometry of training in function space}
	\subsection{Training and the Neural Tangent Kernel}\label{ss:NTK0}
	
	During training, the parameters $\lbrace \theta(t) \rbrace_{t\geq 0}$ evolve through the negative gradient flow of a cost function
	$\cL^{(n)}: \mathbb{R}^{4n} \to \mathbb{R}$, i.e.
	\begin{equation}
		\label{gradientDescent}
		\dot{\theta} = - \nabla \cL^{(n)} ,
	\end{equation}
	where the gradient is calculated
	with respect to the flat, diagonal metric on $\mathbb{R}^{4n}$,
	$$
	\gamma^{(n)} = \sum_{j=1}^{4n} \, d \theta_j^2   \, . 
	$$
	In fact it is the inverse metric that enters the definition of the gradient vector field of a function,
	$$
	\nabla \mathcal{L}^{(n)} \, : \qquad \, \gamma^{(n)} \left(\nabla \,  \mathcal{L}^{(n)} \, , \, \cdot \, \right) =
	d \, \mathcal{L}^{(n)} \, ,
	$$
	or, in local coordinates,
	$$
	\nabla \,  \mathcal{L}^{(n)}  = \sum_{j=1}^{4n} \frac{\partial \,  \mathcal{L}^{(n)}}{\partial \theta_j} \,  \frac{\partial }{\partial \theta_j}   \, .
	$$

	In practice, 
	one associates loss functions
	to training data
	$Z = \lbrace (X_a, Y_a)\rbrace_{a=1}^N$,
	where
	we assume that $X_a \neq X_b$
	for $a \neq b$ and
	each $(X_a,Y_a)$ represents a input output pair, i.e. the function $f$ which we expect to approximate satisfies $f(X_a)=Y_a$ for all $a = 1, \dots, N$.
	One  then 
	considers on $\mathrm{Map}([a,b], \mathbb{R})$ a loss
	$\cL_Z \, : \mathrm{Map}([a,b], \mathbb{R}) \longrightarrow \mathbb{R}$, for example the quadratic loss
	\begin{equation}
		\label{msl}
		\cL_Z(f):=\frac{1}{2} \sum_{i=1}^N \left( Y_i - f(X_i) \right)^2.
	\end{equation}
	Then, for every $n$, 
	we have a loss on $\bbR^{4n}$,
	$$
	\cL^{(n)}_Z = \cL_Z \circ \cF^{(n)} ,
	$$
	and the corresponding gradient flow. Let $\theta(t)$
	be an integral curve of this flow and, for $x \in \bbR$
	consider the curve 
	$y(t) := f^{(n)}_{\theta(t)}(x)$. From $\dot{y} = \sum_{j=1}^{4n} \frac{\partial f^{(n)}_\theta}{\partial \theta_j} (x) \ \dot{\theta}_j $ we find,
		\begin{align*}
			\dot{y} & = \sum_{j=1}^{4n} \frac{\partial f^{(n)}_\theta}{\partial \theta_j} (x) \sum_{l=1}^n \left( Y_l - f^{(n)}_\theta(X_l) \right) \frac{\partial f^{(n)}_\theta}{\partial \theta_j} (X_l) \\
			& = \sum_{l=1}^N  \left( Y_l - f^{(n)}_\theta(X_l) \right) \sum_{j=1}^{4n} \frac{\partial f^{(n)}_\theta}{\partial \theta_j} (x) \frac{\partial f^{(n)}_\theta}{\partial \theta_j} (X_l)\;.
		\end{align*}
		Introducing the neural tangent kernel 
		$$
		\mathrm{NTK}_{z_1 z_2}^{(n)}(\theta) :=  \sum_{j=1}^{4n} \frac{\partial f^{(n)}_\theta}{\partial \theta_j} (z_1) \frac{\partial f^{(n)}_\theta}{\partial \theta_j} (z_2) 
		$$
		we can rewrite the previous computation as
		\begin{equation}\label{eq:Evolution of y}
			\dot{y}  = \sum_{l=1}^N \mathrm{NTK}_{x X_l}^{(n)}(\theta) \left( Y_l - f^{(n)}_\theta(X_l) \right) .
		\end{equation}
		In the setting we are considering, i.e. when $f^{(n)}_\theta$ is given by \ref{eq:f_theta_2}, the neural tangent kernel is
		\begin{equation}
			\label{NTKdef}
			\mathrm{NTK}_{x_1 x_2}^{(n)}(\theta) =  \sum_{j=1}^{4n} \frac{\partial f^{(n)}_\theta}{\partial \theta_j} (x_1) \frac{\partial f^{(n)}_\theta}{\partial \theta_j} (x_2) = \frac{1}{n} \sum_{l,m=1}^n \sum_{j=1}^{4n} \frac{\partial p_l}{\partial \theta_j} (x_1) \frac{\partial p_m}{\partial \theta_j} (x_2) = \frac{1}{n} \left( \sum_{m=1}^n \sum_{j=1}^{4n} \frac{\partial p_m}{\partial \theta_j} (x_1) \frac{\partial p_m}{\partial \theta_j} (x_2) \right)\;,   
		\end{equation}
		because $ \frac{\partial p_l}{\partial \theta_j} (x_1) \frac{\partial p_m}{\partial \theta_j} (x_2)=0$ for $l \neq m$. As a consequence, the law of large numbers ensures that as $n \to + \infty$
		\begin{equation}
			\label{NTKinfty}
			\mathrm{NTK}_{x_1 x_2}^{(n)}(\theta) \to \mathrm{NTK}_{x_1 x_2}^\infty := \mathbb{E} \left[ K(x_1,x_2)  \right]\;, \text{ a.s. in } \theta\;,
		\end{equation}
		where $K(x_1,x_2)$ is the random variable
		$$K(x_1,x_2):=\sum_{j} \frac{\partial p}{\partial \theta_j} (x_1) \frac{\partial p}{\partial \theta_j} (x_2)\;.$$

		In fact it is well known that we have the following quantitative version of the Law of Large Numbers applies~\cite[Theorem 2.5.11]{Durrett}
		\begin{equation}
			\label{NTKconv2}
			\mathrm{NTK}_{x_1 x_2}^{(n)}(\theta)=\mathrm{NTK}_{x_1 x_2}^\infty +o(n^{-1/2}(\log n)^{1/2+}) \;\; \text{ a.s.}\;,
		\end{equation}
		where, as usual, $a+$ represents any number strictly larger than $a$.

		\subsection{Training as the ``pushforward'' of gradient flow to function space}
		
		As in the previous section, let $\lbrace \theta(t) \rbrace_{t \in I \subset \R}$ be an integral curve of the negative gradient of $\cL_Z^{(n)}=\cL_Z \circ \cF^{(n)}$, i.e. $\dot{\theta} = - \nabla \cL^{(n)}_Z$, where $\nabla \cL^{(n)}_Z$ is the vector field on $\R^{4n}$ determined by 
		$$\nabla \cL^{(n)}_Z (\alpha) = (\gamma^{(n)})^{-1}(d\cL^{(n)}_Z , \alpha ) , $$
		for all $1$-forms $\alpha$. Here we have used 
		$$(\gamma^{(n)})^{-1}= \sum_{j=1}^{4n} \frac{\partial}{\partial \theta_j } \otimes \frac{\partial}{\partial \theta_j }$$ 
		to denote the metric on co-vectors induced by the Euclidean metric $\gamma^{(n)}$. Furthermore, as $d \cL_Z^{(n)} = d \cL_Z \circ d\cF^{(n)}$ we actually have
		\begin{eqnarray}\label{eq:gradient L}
			\nabla \cL^{(n)}_Z (\cdot) = (\gamma^{(n)})^{-1}(d\cL_Z \circ d \cF^{(n)} , \cdot ) .
		\end{eqnarray}
		Let us keep this expression at the back of our mind and go on to use $\cF^{(n)}$ to induce a pushforward of the inverse metric along $\lbrace \cF^{(n)}(\theta(t)) \rbrace_{t\in I \subset \R} \subset \rm{Map}([a,b], \R)$, i.e. the image of the integral line we are considering. At any point of $\cF(\theta(t))$ we can define of $\gamma^{(n)}$ as
		$$\left(\cF_\ast (\gamma^{(n)})^{-1} \right)_{\cF^{(n)}(\theta(t))}(\alpha , \beta ) = (\gamma^{(n)})^{-1}_{\theta(t)} \Big( \alpha \circ d\cF^{(n)}_{\theta(t)}  , \beta \circ d\cF^{(n)}_{\theta(t)} \Big) \ ,$$
		for any covectors $\alpha , \beta$ at $\cF^{(n)}(\theta(t))$.\footnote{Here given a co-vector $\alpha$ at $\cF^{(n)}(\theta(t))$, $\alpha \circ d\cF^{(n)}$ denotes the covector at $\theta(t)$ which takes the value $\alpha ( d\cF^{(n)}_{\theta(t)} V)$ when evaluated at a tangent vector $V$.}. Then, using this metric, we can define a gradient of $\cL_Z$ along $\lbrace \cF^{(n)}(\theta(t)) \rbrace_{t\in I \subset \R}$ by
		$$\nabla \cL_Z (\cdot) = \left( \cF^{(n)}_\ast (\gamma^{(n)})^{-1} \right) (d \cL_Z , \cdot) $$
		which unwinding the definition of $\cF^{(n)}_\ast (\gamma^{(n)})^{-1}$ and using that of $\nabla \cL^{(n)}_Z$ in equation \ref{eq:gradient L}, this reads 
		$$(\nabla \cL_Z) (\cdot) =(\gamma^{(n)})^{-1}  \left( d \cL_Z \circ d\cF^{(n)} , (\cdot) \circ d\cF^{(n)} \right) = d\cF^{(n)} ( \nabla \cL_Z^{(n)} (\cdot) ), $$
		leading to the conclusion that
		$$\nabla \cL_Z = d\cF^{(n)} ( \nabla \cL_Z^{(n)} ).$$
		In particular, $\cF^{(n)} (\theta(t))$ evolves with velocity vector given by $-\nabla \cL_Z$, the negative gradient of $\cL_Z$ with respect to the inverse metric $(\gamma^{(n)})^{-1}$. Indeed,
		$$\frac{d}{dt} (\cF^{(n)}(\theta(t))) = d\cF (\dot{\theta}(t) ) = - (d\cF)_{\theta(t)} (\nabla \cL_Z^{(n)})_{\theta(t)} = - ( \nabla \cL_Z)_{\cF(\theta(t)}, $$
		and we have thus obtained a different interpretation for the flow by the neural tangent kernel as a special case of flowing by the negative gradient of $\cL_Z$ with respect to the pushforward of the inverse metric. 
		
		\begin{remark}
			Notice that we only define $ \nabla \cL_Z$ along the image of a gradient flow line. While it would be desirable to define this everywhere in the image of $\cF$, in general that is not possible as this map may fail to be injective in a serious way. 
		\end{remark}

		\subsection{Evolution and convergence of outputs for infinite networks}
		
		In this section we recall some well known facts regarding the evolution of infinitely large neural networks during training. The upshot if that the outputs of the network evolve in a linear manner by the Neural Tangent Kernel, \cite{Lee}.
		
		Given a training set $\{X_i,Y_i\}_{i=1}^N$, we define the function associated to the ``infinite width'' network with kernel  $\mathrm{NTK}^\infty$ (and the same initialization as the finite width one) as the function
		$f_{\infty}(\,\cdot\,,t):\mathbb{R}\rightarrow \mathbb{R}$ defined as the solution to the following initial value problem
		\begin{equation}
			\label{def_fInifinite first}
			\left\{
			\begin{array}{l}
				\partial_t f_{\infty}(x,t) = -\sum_{j=1}^N \mathrm{NTK}^{\infty}_{xX_j}(f_{\infty}(X_j,t)-Y_j)
				\\
				f_{\infty}(x,0)=f_{\theta(0)}^{(n)}(x)\;,
			\end{array}
			\right.
		\end{equation}
		where $f_{\theta(0)}^{(n)}(x)$ denotes the initialized (finite width) network evaluated at $x$. This is a system of ordinary differential equations which can be solved in closed form for all $t \in \mathbb{R}$. The result, which is deduced in the Appendix \ref{appEvolution}, and which we recast here for completeness, follows from first obtaining the linear system for $E_i(t)=f_\infty(X_i,t)-Y_i$, then the evolution equation \ref{def_fInifinite first}, yields
		\begin{align*}
			\dot{E_i} & = -\sum_{l=1}^N \mathrm{NTK}^\infty_{x_i X_l} E_l
		\end{align*}
		whose solution is $E(t)=e^{-\mathrm{NTK}^\infty t } E(0)$ and thus	
		\begin{equation}\label{eq:evolution f infinity at training points}
			f_{\infty}(X_i,t) = Y_i - \sum_{j=1}^N \left( e^{-\mathrm{NTK}^\infty t} \right)_{ij} (f_{\infty}(X_j,0)-Y_j), 
		\end{equation}
		where $\left( e^{-\mathrm{NTK}^\infty t} \right)_{ij}$ denotes the $(i,j)$ entry of the matrix $e^{-\mathrm{NTK}^\infty t} $.

		We now use this to compute the evolution of the evaluation of $f_\infty$ at a general $x \in \mathbb{R}$. As we shall see, and is already implicit in \ref{def_fInifinite first}, its evolution is completely determined by that of $f_\infty$ evaluated at the labels $\lbrace X_i \rbrace_{i=1}^N$. Indeed, using the $\lbrace E_i(t) \rbrace_{i=1}^N$, the evolution equation \ref{def_fInifinite first} can be recast as 
		\begin{align*}
			\partial_t f_{\infty}(x,t) & = - \sum_{l=1}^N \mathrm{NTK}^\infty_{x_{N+1} X_l} E_l(t).
		\end{align*}
		Using the fact that $(\mathrm{NTK}^\infty)_{ij}:=\mathrm{NTK}^\infty_{X_i X_j}$ is invertible \cite{NTK>0}, we can write $E_l(t)=- \sum_{m=1}^N (\mathrm{NTK}^\infty)^{-1}_{lm} \dot{E}_m$ which inserting above gives 
		\begin{align*}
			\partial_t f_\infty(x,t) & =  \sum_{l,m=1}^N \mathrm{NTK}^\infty_{x_{N+1} X_l} (\mathrm{NTK}^\infty)^{-1}_{lm} \dot{E}_m(t).
		\end{align*}
		We may then use the fundamental theorem of calculus to obtain
		\begin{align*}
			f_\infty(x,t) & =f_\infty(x,0) + \int_0^t \sum_{l,m=1}^N \mathrm{NTK}^\infty_{x X_l} (\mathrm{NTK}^\infty)^{-1}_{lm} \dot{E}_m(s) ds \\
			& = f_\infty(x,0) +  \sum_{l,m=1}^N \mathrm{NTK}^\infty_{x X_l} (\mathrm{NTK}^\infty)^{-1}_{lm} (E_m(t) - E_m(0)) \\
			& = f_\infty(x,0) +  \sum_{l,m=1}^N \mathrm{NTK}^\infty_{x X_l} (\mathrm{NTK}^\infty)^{-1}_{lm} (f_\infty(X_m,t) - f_\infty(X_m,0)) ,
		\end{align*}
		and using \ref{eq:evolution f infinity at training points}, we find
		\begin{equation}\label{eq:Infinite width network}
			f_\infty(x,t)  = f_\infty(x,0) +  \sum_{l,m,p=1}^N \mathrm{NTK}^\infty_{x X_l} (\mathrm{NTK}^\infty)^{-1}_{lm} (e^{-\mathrm{NTK}^\infty t} -1)_{mp} (f_\infty(X_p,0)-Y_p) ,
		\end{equation}
		which is the evolution, through training, of an infinite width network initialized as $f_\infty(x,0)$.
		
		\begin{remark}
			In particular, we find that
			$$\lim_{t \to + \infty} f_\infty(x,t) = f_\infty(x,0) - \sum_{l,m=1}^N \mathrm{NTK}^\infty_{x_{N+1} X_l} (\mathrm{NTK}^\infty)^{-1}_{lm} (f_\infty(X_m,0)-Y_m),$$
			which should be compared with  kernel regression.
		\end{remark}

		\subsection{Evolution and convergence of outputs for wide, but finite, networks}
		\label{secEvolutionOutputs}
		
		In this section we present some results concerning the the evolution of the outputs of a width $n$ network. As was already discussed this evolution is determined by~\eqref{eq:Evolution of y}, and so it is intimately related to, the evolution of $\mathrm{NTK}^{(n)}(\theta(t))$. Moreover, when $n\rightarrow\infty$, this finite width NTK converges, almost surely, to $\mathrm{NTK}^{\infty}$, which is constant in parameter space. Therefore, in the infinite width limit, the evolution of outputs is determined by a linear flow. We will explore this fact (see Appendix~\ref{appEvolution} for the proofs) to control the non-linear learning dynamics of finite width networks in terms of linear infinite width dynamics. This is achieved via a bootstrap/continuity argument.

		We are now able to state our main results concerning the evolution of outputs: 
		\begin{theorem}
			\label{thmEvolutionComplete}
			Consider a sequence $\theta_{0}=(W^{(2)}_{0,j},b^{(2)}_{0,\mu},W^{(1)}_{0,\mu},b^{(1)}_{0,\mu})_{\mu\in\mathbb{N}}$, of 
			independent and identically distributed tuples of random variables with finite third order momenta. Consider also 
			a non-polynomial activation function satisfying 
			\begin{equation}
				\label{locBounded0}
				|\sigma(z)|\leq C_{\sigma}(|z|+1) \;\;\text{ and }\;\;|\sigma'(z)|+|\sigma''(z)|\leq C_{\sigma}\;,
			\end{equation}
			for some $C_{\sigma}>0$, and all $z\in\mathbb{R}$. 
			For any collection of $4n$ parameters $\theta=(W^{(2)}_{i},b^{(2)}_{i},W^{(1)}_{i},b^{(1)}_{i})_{i=1,\dots, n}$, let  
			$$f^{(n)}_{\theta} (x) = \frac{1}{\sqrt{n}} \sum_{j=1}^{n} \left( W^{(2)}_{j} \ \sigma \left(  W^{(1)}_{j} x +b_{j}^{(1)} \right) + b^{(2)}_{j} \right)$$
			be the corresponding width $n$ neural network. 
			
			
			Assume also that we are given a training set $\{X_i,Y_i\}_{i=1}^N$, with $X_i\neq X_j$, for $i\neq j$, and let $\lambda_{\infty}>0$ be the minimum eigenvalue of the matrix $\mathrm{NTK}_{X_{i}X_{j}}^\infty$.~\footnote{The positivity of $\lambda_{\infty}$ follows from the results in~\cite{NTK>0}.}
			
			
			Let $t\mapsto \theta(t)$ be the parameter evolution during training, initialized with $\theta(0)=\theta_{0}$, and training performed by gradient descent~\eqref{gradientDescent} with mean squared loss~\eqref{msl}.     

			For any given $\delta>0$, there exists $m>0$ and  $C_{\delta}>0$ such that, if $n\geq m$, then with probability $1-\delta$ (at the initialization of the parameters): we have, for all $t\geq0$, 
			
			\begin{enumerate}[(i)]
				\item {\em Difference between the finite width NTK and its infinite limit:}
				\begin{align}
					|\mathrm{NTK}_{zw}^{(n)}(\theta(t))-\mathrm{NTK}_{zw}^\infty|
					&\leq 
					\frac{C_{\delta}\;(\log n)^{1/2+}}{\sqrt{n}} + 
					\sup_{k=1,\dots,N}|f^{(n)}_{\theta(0)}(X_k)-Y_k|\frac{NC}{\sqrt{n}(\lambda_{\infty}-\epsilon)}\;,
				\end{align}
				for all $z,w \in [a,b]$ and where $\epsilon>0$ can be made arbitrarily small by increasing $n$.
				\item 
				{\em Exponential convergence of training outputs (of the finite width neural network) to their labels:}
				\begin{equation}
					\label{decayError0}
					|f^{(n)}_{\theta(t)}(X_i)-Y_i|\leq C\sup_{l=1,\dots,N}|f^{(n)}_{\theta(0)}(X_l)-Y_l|  e^{-(\lambda_{\infty}-\epsilon )t}\;, 
				\end{equation}
				where $\epsilon>0$ can be made arbitrarily small by increasing $n$.

				\item {\em Stability of training:} 
				let $f_{\infty}$ be the function associated to the infinite width Kernel (as defined by~\eqref{eq:Infinite width network}), 
				then, for all $x \in [a,b]$, 
				\begin{eqnarray}
					|f^{(n)}_{\theta(t)}(x)-f_{\infty}(t,x)|&\leq &   \left(1+\sup_{i=1,\ldots,N}|\mathrm{NTK}_{xX_{i}}^\infty|\right)
					\frac{C E_0}{\lambda_{\infty}-\epsilon}  \frac{(\log(n))^{1/2+}}{\sqrt{n}}
					\;,
				\end{eqnarray}
				where $E_0:=\sup_{k=1,\dots,N}|f^{(n)}_{\theta(0)}(X_k)-Y_k|$ and $\epsilon>0$, as before\;.
			\end{enumerate}
		\end{theorem}
		
		An immediate corollary of this result is the following.
		
		\begin{corollary}[Exponential convergence of the network to its limit]\label{cor:exponential convergence}
			In the conditions of Theorem \ref{thmEvolutionComplete}, there is a positive constant $C$ such that for all $x \in  \mathbb{R}$, we have
			\begin{equation}\label{eq:exponential convergence}
				|\lim_{s \to + \infty} f^{(n)}_{\theta(s)}(x) - f^{(n)}_{\theta(t)}(x)|   \leq  \frac{C\sup_{l=1,\dots,N}|f^{(n)}_{\theta(0)}(X_l)-Y_l|}{\lambda_\infty - \epsilon}  e^{-(\lambda_{\infty}-\epsilon )t}, 
			\end{equation}
		\end{corollary}
		\begin{proof}
			Equation \ref{decayError0} shows that for sufficiently large $n$, and $i=1, \ldots N$, we have 
			$$|f^{(n)}_{\theta(t)}(X_l)-Y_l|\leq C\sup_{k=1,\dots,N}|f^{(n)}_{\theta(0)}(X_k)-Y_k|  e^{-(\lambda_{\infty}-\epsilon )t}$$
			We can use this estimate to further compute the difference between $f^{(n)}_{\theta(t)}(x)$ and its limit. This requires using the evolution equation for $f_{\theta(t)}^{(n)}(x)$ which was deduced in equation \ref{eq:Evolution of y} and reads 
			$$\frac{\partial}{\partial t} \left( f_{\theta(t)}^{(n)}(x) \right) = - \sum_{i=1}^N \mathrm{NTK}_{x X_i} (f_{\theta(t)}^{(n)}(X_i) - Y_i).$$ 
			Then, by possibly changing the constant $C$, gives
			\begin{align*}
				|\lim_{s \to + \infty} f^{(n)}_{\theta(s)}(x) - f^{(n)}_{\theta(t)}(x)| & = |\int_t^{+\infty} \partial_t f^{(n)}_{\theta(s)}(x) ds| \\
				& \leq C\sup_{l=1,\dots,N}|f^{(n)}_{\theta(0)}(X_l)-Y_l|  \int_t^{+\infty} e^{-(\lambda_{\infty}-\epsilon )s} ds \\
				& \leq  \frac{C\sup_{l=1,\dots,N}|f^{(n)}_{\theta(0)}(X_l)-Y_l|}{\lambda_\infty - \epsilon}  e^{-(\lambda_{\infty}-\epsilon )t}, 
			\end{align*}
			which shows that $f^{(n)}_{\theta(t)}(x)$ exponentially converges to its limit as promised in the statement. 
		\end{proof}
		
		
		
		

		Furthermore, we point out that there are interesting results in the realm of deep neural networks which guarantee that the loss function converges to zero along training. See for instance \cite{Du} and also \cite{Chizat}.

		\section{Approximating the evolution during training}\label{sec:Evolution during training}

		\subsection{Evolution by the flow of a vector field}
		
		As before, given $x_1 , \ldots , x_k \in \mathbb{R}$ we consider the evaluation of a network in these points $(f_\theta(x_1), \ldots , f_\theta(x_k)) \in \mathbb{R}^k$. For simplicity and clarity, we shall parameterize the space of outputs, $\mathbb{R}^k$, using coordinates $(y_1, \ldots , y_k)$. Furthermore, when appropriate, we shall simplify notation by using $x=(x_1, \ldots , x_k)$ and $y=(y_1, \ldots , y_k)$. 
		
		Formally speaking it is easy to say how $y(t)$ evolves during training. Indeed, denote by $\lbrace \phi_t^{(n)} \rbrace_{t \geq 0}$ the flow of $-\nabla_\theta C_\theta$ in $\mathbb{R}^{4n}$, i.e.  $\lbrace \theta(t)= \phi_t^{(n)}(\theta) \rbrace_{t \geq 0}$ solves $\dot{\theta}=-\nabla_\theta C_\theta$, then $y(t)=f_{\phi_t^{(n)} ( \theta)}$ evolves through \ref{eq:Evolution of y}. 
		Consider the induced measure at initialization $\mathbb{P}_{x_1, \ldots ,x_k}^{(n)}$, as defined in \ref{eq:Finite dimensional pushforward measures}. Unfortunately, though correct, this is not very helpful in explicitly understanding how the probability distribution 
		\begin{equation}\label{eq:measure along training for finite n}
			\mathbb{P}^{(n)}_{x_1, \ldots , x_k}(t) :=  \left(  {\mathcal F}^{(n)}_{x_1, \ldots , x_k} \circ \phi_t \right)_* (Q_n),
		\end{equation}  
		evolves with time. For this it would be very convenient if $\lbrace y(t) \rbrace_{t \geq 0}$ was evolving by the flow of a vector field obtained by the pushforward of $- \nabla_\theta C_\theta$ via ${\mathcal F}^{(n)}_{x_1, \ldots , x_k}: \mathbb{R}^{4n} \to \mathbb{R}^k$. However, given that ${\mathcal F}^{(n)}_{x_1, \ldots , x_k}$ is not in general injective there is no way to compatibly define such a pushforward vector field.
		
		One possible way to overcome this difficulty, at least formally, is to instead consider the graph of ${\mathcal F}^{(n)}_{x_1, \ldots , x_k}$ and define a vector field there
		$$\mathrm{graph}( {\mathcal F}^{(n)}_{x_1, \ldots , x_k} ) = \lbrace  (\theta, y) = ( (\theta_1 , \ldots ,\theta_P), (y_1, \ldots , y_k) ) \in \mathbb{R}^P \times \mathbb{R}^k \ | \ y_i=f_\theta(x_i) \ \text{for $i=1, \ldots , k$} \rbrace .$$ 
		Then, the map
		$$\varphi_{x_1 , \ldots , x_k} : \mathbb{R}^{4n} \to \mathrm{graph}( {\mathcal F}^{(n)}_{x_1, \ldots , x_k} )$$
		given by
		$$\varphi_{x_1 , \ldots , x_k}(\theta_1 , \ldots , \theta_{4n})= ((\theta_1 , \ldots , \theta_{4n}), (f_\theta(x_1), \ldots , f_\theta(x_k) ))$$
		is a diffeomorphism and so we can define the vector field $(\varphi_{x_1 , \ldots , x_k})_* (- \nabla_\theta C_\theta)$. By writing $- \nabla_\theta C_\theta = - \sum_j \frac{\partial C_\theta}{\partial \theta_j} \frac{\partial}{\partial \theta_j}$ it follows from equation \ref{eq:Evolution of y} that
		$$(\varphi_{x_1 , \ldots , x_k})_* (- \nabla_\theta C_\theta) = - \sum_{j=1}^{4n} \frac{\partial C_\theta}{\partial \theta_j} \frac{\partial}{\partial \theta_j} + \sum_{i=1}^k \left( \sum_{l=1}^N \mathrm{NTK}_{x_i X_l}^{(n)}(\theta) ( Y_l - f_\theta(X_l) )  \right) \frac{\partial}{\partial y_i} . $$
		Then, the measure $(\varphi_{x_1 , \ldots , x_k} \circ \phi_t)_*(Q_n)$ in $\mathrm{graph}({\mathcal F}^{(n)}_{x_1, \ldots , x_k})$ evolves through the flow of the vector field $V$ and we can obtain $Q^{(n)}_{x_1, \ldots , x_k}(t)$ simplify pushing it forward along the projection on the second factor $\pi: \mathrm{graph}(\psi_{x_1 , \ldots , x_k}) \subset \mathbb{R}^{4n} \times \mathbb{R}^k \to \mathbb{R}^k$.
		
		Then, $f^{(n)}_\theta(x_i)=y_i$ and so
		$$(\varphi_{x_1 , \ldots , x_k})_* (- \nabla_\theta C_\theta) = - \sum_{j=1}^P \frac{\partial C_\theta}{\partial \theta_j} \frac{\partial}{\partial \theta_j} + \sum_{i=1}^k \left( \sum_{l=1}^N \mathrm{NTK}_{x_i X_l}^{(n)}(\theta) ( Y_l - y_l )  \right) \frac{\partial}{\partial y_i} . $$
		\begin{remark}\label{rem: Vector Field in the Inifinite width limit}
			As already discussed, in this situation, when $n \to + \infty$, $\mathrm{NTK}_{z_1 z_2}^{(n)}(\theta) \to \mathrm{NTK}_{z_1 z_2}^\infty$ with the limit being a.s. independent of $\theta$, and so we can indeed define a vector field on $\mathbb{R}^k$ by
			$$V=\sum_{i=1}^k \left( \sum_{l=1}^N \mathrm{NTK}_{x_i X_l}^\infty \left( Y_l - y_l \right)  \right) \frac{\partial}{\partial y_i}. $$
			In practice, given the uncertainties involved it turns out that for large $n$ it may be a good idea to use $V$ to compute the first order deviations in the evolution. This will be done in the next section.
		\end{remark}

		\subsection{Evolution in the infinite width case}

		As noted in remark \ref{rem: Vector Field in the Inifinite width limit}, in the infinite width limit, we can define a vector field $V$ on the space of outputs, $\mathbb{R}^k$ with coordinates $(y^1, \ldots , y^k)$, whose $i$-th entry is
		$$V_i:=\sum_{l=1}^N \mathrm{NTK}_{x_i X_l}^\infty \left( Y_l - f_\theta(X_l) \right).$$
		This has the property that, in this infinite width limit, training corresponds to the flow of $V$. As before, we denote its flow by $\lbrace \Phi_t^{\infty} \rbrace_{t \geq 0}$ and define the density, $\nu_{x_1, \ldots, x_k}^\infty(t,y_1 , \ldots, y_k)$ of the measure 
		\begin{equation}\label{eq:measure along training for infinite n}
			\mathbb{P}_{x_1, \ldots, x_k}^\infty (t):= \lim_{n \to \infty } \mathbb{P}_{x_1, \ldots, x_k}^{(n)} (t),
		\end{equation}
		by
		$$\nu_{x_1, \ldots, x_k}^\infty(t,y_1 , \ldots, y_k) dy^1 \ldots dy^k = d\mathbb{P}_{x_1, \ldots, x_k}^\infty (t).$$
		As for the evolving densities, they can be equivalently defined using the pullback via $\Phi_{-t}^{\infty}$, i.e.
		$$\nu^\infty_{x_1 , \ldots , x_k}(t,y_1, \ldots , y_k) dy^1 \ldots dy^k = (\Phi_{-t}^{\infty})^*(\nu^\infty_{x_1 , \ldots , x_k}(0,y_1, \ldots , y_k) dy^1 \ldots dy^k).$$ 
		Using the shorthand notation $x=(x_1,\ldots ,x_k)$ and $y=(y_1, \ldots , y_k)$, we can explicitly write 
		\begin{equation}\label{eq: Evolution of the density (general formula)}
			\nu^\infty_x(t,y) = \nu^\infty_x(0, \Phi_{-t}^{\infty}(y)) |\det (d \Phi_{-t}^{\infty})|,
		\end{equation}
		from which we find that $\nu^\infty_x(t,y)$ solves the continuity equation
		$$\partial_t \nu^\infty_x + \mathrm{div}_y (\nu^\infty_x V) =0,$$
		with the initial condition $\nu^\infty_x(0,y)$ given by equation \ref{eq:Gaussian Limit}. 
		
		\begin{remark}
			Sometimes, it is convenient to work with the Fourier transform, i.e. the characteristic function associated with $\nu^\infty_x(t,y)$. We denote this by
			$$Z^\infty_{\nu_x}(t,p) = \int_{\mathbb{R}^k} e^{i \sum_{j=1}^k p_j y_j} \nu^\infty_x(t,y)dy.$$
			Computing its evolution in time gives
			\begin{align*}
				\partial_t Z^\infty_{\nu_x}(t,p) & =  \int_{\mathbb{R}^k} e^{i \sum_{j=1}^k p_j y_j} \partial_t \nu^\infty_x(t,y)dy = - \int_{\mathbb{R}^k} e^{i \sum_{j=1}^k p_j y_j} \mathrm{div}_y (\nu^\infty_x V)  dy,
			\end{align*}
			and integrating by parts yields
			\begin{align}\nonumber
				\partial_t Z^\infty_{\nu_x}(t,p) & = - \int_{\mathbb{R}^k} \left[ \mathrm{div}_y ( e^{i \sum_{j=1}^k p_j y_j} \nu^\infty_x V) - i \sum_{l=1}^k p_l V_l e^{i \sum_{j=1}^k p_j y_j} \nu^\infty_x \right] dy \\ \label{eq:Evolution for Z_nu}
				& =  i \sum_{l=1}^k p_l \int_{\mathbb{R}^k}   V_l e^{i \sum_{j=1}^k p_j y_j} \nu^\infty_x  dy .
			\end{align}
			In many situations, this can be used to effectively compute the evolution of $Z^\infty_{\nu_x}(t,p)$.  Just for the sake of exemplifying this, consider for simplicity, that $V$ does not depend on $y$. Then, the above evolution equation turns into
			$$\partial_t Z_{\nu_x}(t,p) =  i \sum_{l=1}^k p_l V_l Z_{\nu_x}(t,p), $$
			which we can integrate to
			$$Z_{\nu_x}(t,p)= Z_{\nu_x}(0,p) \exp \left(  i \sum_{l=1}^k p_l \int_0^t V_l (s) ds \right).$$
			Hence, in order to find $\nu_x(t,y)$ we need only apply the inverse Fourier transform, yielding
			\begin{align*}
				\nu_x(t,y) & = \frac{1}{(2\pi)^k} \int_{\mathbb{R}^k} Z_{\nu_x}(t,p) e^{-i\sum_{j=1}^k y_j p_j} dp \\
				& = \frac{1}{(2\pi)^k} \int_{\mathbb{R}^k} Z_{\nu_x}(0,p) e^{-i\sum_{j=1}^k  p_j  \left( y_j -  \int_0^t V_j (s) ds \right) } dp \\
				& = \nu_x \left(0,y - \int_0^t V (s) ds \right) .
			\end{align*}
		\end{remark}
		
		We start be considering the important case when $k=N$ and $x_i=X_i$ for $i=1, \ldots , N$. In this situation, we can write $V$ as
		$$V = \sum_{i=1}^N \left( \sum_{l=1}^N \mathrm{NTK}_{X_i X_l}^\infty \left( Y_l - y_l \right) \right) \frac{\partial}{\partial y_i}. $$
		In this situation we can explicitly compute the flow $\lbrace \Phi_t^{\infty} \rbrace_{t \geq 0}$. Indeed, this was already done in equation \ref{eq:evolution f infinity at training points}, where $y_i$ is replaced by $f_{\infty}(X_i,t)$. The result is
		$$\Phi_t^{\infty}(y)=Y + e^{-\mathrm{NTK}^\infty t } (y-Y),$$
		and using the general formula for the evolution of the density, equation \ref{eq: Evolution of the density (general formula)}, we find that
		\begin{align}\nonumber
			\nu^\infty_x(t,y) & = \nu^\infty_x(0, \Phi_{-t}^{\infty}(y)) |\det (d \Phi_{-t}^{\infty})| \\ \label{eq:nu example k=N}
			& = \nu^\infty_x( 0 , Y + e^{\mathrm{NTK}^\infty t } (y-Y)) \det (  e^{\mathrm{NTK}^\infty t } )
		\end{align}
		
		\begin{remark}
			We can confirm this by solving \ref{eq:Evolution for Z_nu}. Indeed, we have
			\begin{align*}
				\partial_t Z_{\nu_x}^\infty(t,p) & =  i \sum_{i=1}^N p_i \int_{\mathbb{R}^k}   V_i e^{i \sum_{j=1}^k p_j y_j} \nu_x^\infty  dy \\
				& =  i \sum_{i=1}^N p_i \int_{\mathbb{R}^k} \left( \sum_{l=1}^N \mathrm{NTK}_{X_i X_l}^\infty \left( Y_l - y_l \right) \right) e^{i \sum_{j=1}^k p_j y_j} \nu_x^\infty  dy \\
				& =  i \sum_{i,l=1}^N p_i \ \mathrm{NTK}_{X_i X_l}^\infty  \left( Y_l Z_{\nu_x}^\infty(t,p) - \int_{\mathbb{R}^k}y_l \nu_x^\infty \ e^{i \sum_{j=1}^k p_j y_j}  dy \right) \\
				& =  i \sum_{i,l=1}^N p_i \ \mathrm{NTK}_{X_i X_l}^\infty  \left( Y_l Z_{\nu_x}^\infty(t,p) + i \partial_{p_l} Z_{\nu_x}^\infty(t,p)  \right) .
			\end{align*}
			We can rewrite this as $\left( \partial_t - i \sum_{i,l=1}^N p_i \ \mathrm{NTK}_{X_i X_l}^\infty  \left( Y_l  + i \partial_{p_l} \right)  \right) Z_{\nu_x}^\infty(t,p) = 0 $, or
			$$\left( \partial_t + \sum_{i,l=1}^N p_i \ \mathrm{NTK}_{X_i X_l}^\infty  \ \partial_{p_l} - i \sum_{i,l=1}^N p_i \ \mathrm{NTK}_{X_i X_l}^\infty  \ Y_l   \right) Z_{\nu_x}^\infty(t,p) = 0 .$$
			This equation can be solved using the method of characteristics. Indeed, let $T \in \mathbb{R} \mapsto P(T) \in \mathbb{R}^N$ be a path solving the equation
			$$\dot{P}_j(T)=\sum_{i=1}^N \mathrm{NTK}_{X_i X_j}^\infty P_i(T),$$
			i.e. $P(T)=\exp(\mathrm{NTK}^\infty T) P(0)$.
			Then, the function $U(T):=Z_{\nu_x}(T,P(T))$ evolves through

			\[
			\dot{U}(T)  = \partial_t Z_{\nu_x}^\infty (T,P(T)) + \sum_{j=1}^N \partial_{p_j}Z_{\nu_x}^\infty(T,P(T)) \dot{P_j}(T)
			\]
			then,
			\begin{align*}
				\dot{U}(T) & = - \sum_{i,l=1}^N P_i(T) \ \mathrm{NTK}_{X_i X_l}^\infty  \ \partial_{p_l} Z_{\nu_x}^\infty (T,P(T)) + i \sum_{i,l=1}^N P_i(T) \ \mathrm{NTK}_{X_i X_l}^\infty  Y_l  Z_{\nu_x}^\infty (T,P(T)) \ + \\
				\ \ \ \ & + \sum_{j=1}^N \partial_{p_j}Z_{\nu_x}^\infty (T,P(T)) \sum_{i=1}^N \mathrm{NTK}_{X_i X_j}^\infty P_i(T) \\
				& = - i \sum_{i,l=1}^N P_i(T) \ \mathrm{NTK}_{X_i X_l}^\infty  Y_l  Z_{\nu_x}^\infty (T,P(T)) ,
			\end{align*}
			and using $Z_{\nu_x}(T,P(T))=U(T)$ and $P(T)=e^{-\mathrm{NTK}^\infty T}P(0)$, the ODE above can be integrated to
			$$U(T)= U(0) \exp \left( i \int_0^T \sum_{i,j=1}^N (e^{\mathrm{NTK}^\infty s}P(0))_i \mathrm{NTK}^\infty_{X_i X_j} Y_j  ds \right).$$
			However, notice that
			$$\sum_{i,j=1}^N (e^{\mathrm{NTK}^\infty s}P(0))_i \mathrm{NTK}^\infty_{X_i X_j} Y_j =  \frac{d}{ds} \langle e^{\mathrm{NTK}^\infty s } P(0), Y \rangle_{\mathbb{R}^N} , $$
			and inserting above yields
			$$U(T)= U(0) \exp \left[ i  \left( \langle e^{\mathrm{NTK}^\infty T } P(0), Y \rangle_{\mathbb{R}^N} - \langle P(0), Y \rangle_{\mathbb{R}^N} \right) \right],$$
			or in other words
			$$Z_{\nu_x}^\infty (T,e^{\mathrm{NTK}^\infty T}P(0))= Z_{\nu_x}^\infty (0,P(0)) \exp \left[ i \left( \langle e^{\mathrm{NTK}^\infty T } P(0)  , Y \rangle_{\mathbb{R}^N} - \langle  P(0), Y \rangle_{\mathbb{R}^N} \right) \right].$$
			This can finally be rearranged to give the final expression of $Z_{\nu_x}^\infty (t,p)$ which is
			$$Z_{\nu_x}^\infty (t,p)= Z_{\nu_x}^\infty (0, e^{-\mathrm{NTK}^\infty t }p ) \exp \left[ i \left( \langle p , Y \rangle_{\mathbb{R}^N} - \langle e^{-\mathrm{NTK}^\infty t }p , Y \rangle_{\mathbb{R}^N} \right) \right].$$
			We can then go back to compute the density measure $\nu_x$ using the inverse Fourier transform
			\begin{align*}
				\nu_x^\infty (t,y) & = \frac{1}{(2\pi)^N} \int_{\mathbb{R}^N} Z_{\nu_x}^\infty (t,p) e^{-i \langle y , p \rangle } dp \\
				& = \frac{1}{(2\pi)^N} \int_{\mathbb{R}^N} Z_{\nu_x}^\infty (0, e^{-\mathrm{NTK}^\infty t }p ) \exp \left[ -i \left( \langle e^{-\mathrm{NTK}^\infty t }p , Y \rangle - \langle p, Y \rangle + \langle p, y \rangle \right) \right]  dp \\
				& = \frac{1}{(2\pi)^N} \int_{\mathbb{R}^N} Z_{\nu_x}^\infty (0, q ) \exp \left[ -i  \langle  q ,  Y - e^{\mathrm{NTK}^\infty t } (Y - y) \rangle \right] \det(e^{ \mathrm{NTK}^\infty t }) \ dq \\
				& = \det(e^{ \mathrm{NTK}^\infty t }) \ \nu_x(0, Y - e^{ \mathrm{NTK}^\infty t } (Y - y) ),
			\end{align*}
			which, as expected, is in agreement with \ref{eq:nu example k=N}.
		\end{remark}
		
		In this situation, when $k=N$ and $x_i=X_i$ for $i=1, \ldots , N$, in order to determine the evolving density $\nu^\infty_x(t,y)$ we must insert into \ref{eq:nu example k=N} the formula for the initial density given in equation \ref{eq:Gaussian Limit}.
		This follows from precomposing \ref{eq:Gaussian Limit} with $\Phi_{-t}^{\infty}$, which requires computing
		\begin{align*}
			(C_{ij}^{-1} y_i y_j ) \circ \Phi_{-t}^{\infty} = A(t)_{ij} (y_i + (e^{-\mathrm{NTK}^\infty t} -1)_{im} Y_m) (y_j + (e^{-\mathrm{NTK}^\infty t} -1)_{jp} Y_p) ,
		\end{align*}
		where
		$$A(t)_{ij}= \left(e^{\mathrm{NTK}^\infty t} \right)_{il} C^{-1}_{lp} \left(e^{\mathrm{NTK}^\infty t} \right)_{pj},$$ 
		and
		$$\frac{\det(e^{ \mathrm{NTK}^\infty t })}{\sqrt{\det(C)}}= \frac{1}{\sqrt{\det(A(t)^{-1})}}.$$
		Finally, inserting these into \ref{eq:nu example k=N}, gives
		\begin{equation}\label{eq:density k=N}
			\nu^\infty_x(t,y) = \frac{e^{- A(t)_{ij} (y_i + (e^{-\mathrm{NTK}^\infty t} -1)_{im} Y_m) (y_j + (e^{-\mathrm{NTK}^\infty t} -1)_{jp} Y_p)}}{ \sqrt{(2\pi)^{k} \det(A(t)^{-1})} } .
		\end{equation}
		
		\begin{remark}
			From the expression \ref{eq:density k=N}, and the fact that $\mathrm{NTK}^\infty$ is positive definite \cite{NTK>0}, we find that
			$$\lim_{t \to + \infty} \nu^\infty_x(t,y) dy= \prod_{i=1}^N \delta (y_i-Y_i) dy,$$
			weakly.
		\end{remark}
		
		We now consider the next situation, in which $k=N+1$ and $x_i=X_i$ for $i=1, \ldots N$ and $x_{N+1} \in \mathbb{R}$ an extra arbitrary point. The evolution of $y_{N+1}$, i.e. $(\Phi_t^{\infty}(y))_{N+1}=f_\infty(x,t)$, was already computed in \ref{eq:Infinite width network}. The result we found then, is that
		$$	(\Phi_t^{\infty}(y))_{N+1} = y_{N+1} +  \sum_{l,m,p=1}^N \mathrm{NTK}^\infty_{x_{N+1} X_l} (\mathrm{NTK}^\infty)^{-1}_{lm} (e^{-\mathrm{NTK}^\infty t} -1)_{mp} (y_p -Y_p) .$$
		In particular, notice that $\partial_{y^{N+1}} (\Phi_t^{\infty}(y))_{N+1}=1$ while $\partial_{y^{N+1}}  (\Phi_t^{\infty}(y))_{i}=0$ for $i \in 1, \ldots ,N$, thus showing that $\det(d\Phi_t^{\infty})=\det (e^{ \mathrm{NTK}^\infty t })$ as before. We find that for $y=(y_1, \ldots , y_N)$ and $Y=(Y_1, \ldots , Y_N)$
		\begin{align*}
			\nu^\infty_x(t,y , y_{N+1}) & = \nu^\infty_x \left( 0, Y + e^{\mathrm{NTK}^\infty t } (y-Y) , y_{N+1} +  \mathrm{NTK}^\infty_{x_{N+1} X_l} (\mathrm{NTK}^\infty)^{-1}_{lm} (e^{\mathrm{NTK}^\infty t} -1)_{mp} (y_p-Y_p) \right) \times \\
			& \ \ \ \times \det(e^{ \mathrm{NTK}^\infty t }),
		\end{align*}
		with a sum over the repeated indices $l,m,p$ from $1$ to $N$.

		Finally, we consider the general case with $k > N$. Given that in the previous case $x_{N+1}$ was arbitrary, the computations  that led to \ref{eq:Infinite width network} hold true for any $x \in \mathbb{R}$. From now on, we shall set $x_1 =X_1 , \ldots , x_{N}=X_N $, and let $ x_{N+1}  ,\ldots  x_{N+k} \in \mathbb{R}$ be arbitrary, and, in order to ease notation, we shall let $X=(X_1, \ldots , X_N)$ and $x=(x_{N+1}, \ldots , x_{N+k})$. Then, for $i=1, \ldots , N$ we have
		$$	(\Phi_t^{\infty}(y))_{i} = Y_i + \sum_{m=1}^N (e^{ - \mathrm{NTK}^\infty t })_{im} (y_m - Y_m), $$
		while for $i=N+1 , \ldots N+k$
		$$	(\Phi_t^{\infty}(y))_{i} = y_{i} +  \sum_{l,m,p=1}^N \mathrm{NTK}^\infty_{x_{i} X_l} (\mathrm{NTK}^\infty)^{-1}_{lm} (e^{-\mathrm{NTK}^\infty t} -1)_{mp} (y_p -Y_p) . $$
		In fact, one can readily check that the first equation is actually a particular case of the second and so we may as well only work with the second. To do this, we shall view $y\in \mathbb{R}^k$ as a column vector, $\pi_N (y)=(y_1, \ldots , y_N) \in \mathbb{R}^N$ its projection on the first $N$ coordinates, $\mathrm{NTK}^{\infty}_x$ as the $k \times N$ matrix with entries $\mathrm{NTK}^{\infty}_{x_i X_j}$ for $i \in \lbrace 1, \ldots , k \rbrace$ and $j \in \lbrace 1 , \ldots , N \rbrace$, and $\mathrm{NTK}^{\infty}$ as the $N \times N$ matrix with entries $\mathrm{NTK}^{\infty}_{X_i X_j}$ for $i,j \in \lbrace 1, \ldots , N \rbrace$ just as we have been doing up to here. This leads to the general evolution equation 
		\begin{equation}\label{eq:Phi infinity}
			\Phi_t^{\infty}(y) = y +   \mathrm{NTK}^\infty_{x} (\mathrm{NTK}^\infty)^{-1} (e^{-\mathrm{NTK}^\infty t} -1) (\pi_N(y) -Y) . 
		\end{equation}
		Using a similar argument as before, we have $\det(d \Phi_t^{\infty}) = \det (e^{\mathrm{NTK}^\infty t})$ and we can compute $\nu^{\infty}_{x}(t,y)  = \nu^\infty_x(0,\Phi_{-t}^{\infty}(y)) |\det (d \Phi_{-t}^{\infty})|$ in the same way. This leads to the following result.
		
		\begin{theorem}\label{prop: density for infinite with along training}
			Let $k \geq N$ and $x=(x_1, \ldots , x_k)$ with $x_1=X_1, \ldots , x_N=X_N$ the training points. Then, the density $\nu^{\infty}_{x}(t,y)$ of the measure $\mathbb{P}_x^{\infty}(t)$ induced on $\mathbb{R}^k$ satisfies
			$$\nu_x^{\infty}(t,y) = \frac{e^{- \langle \Phi^\infty_{-t}(y) , C^{-1} \Phi^\infty_{-t}(y) \rangle }}{\sqrt{(2\pi)^k \det(A(t)^{-1})  }},$$
			where
			$$	(\Phi_{-t}^{\infty}(y))_{i} = y_{i} +  \sum_{l,m,p=1}^N \mathrm{NTK}^\infty_{x_{i} X_l} (\mathrm{NTK}^\infty)^{-1}_{lm} (e^{\mathrm{NTK}^\infty t} -1)_{mp} (y_p -Y_p) , $$
			for $i=1, \ldots , k$, and
			$$A(t)_{ij}= \left(e^{\mathrm{NTK}^\infty t} \right)_{il} C^{-1}_{lp} \left(e^{\mathrm{NTK}^\infty t} \right)_{pj},$$
			for $i,j=1, \ldots , k$.
		\end{theorem}
		
		\begin{remark}
			If we are only interested on the distribution of $(y_{N+1}, \ldots , y_{N+k})$, this can be computed from that of $(y_1, \ldots , y_N , y_{N+1}, \ldots , y_{N+k})$ by integrating over the first $N$ variables, i.e.
			\begin{align*}
				\nu_{x}(t,y_{N+1}, \ldots , y_{N+k}) = \int_{\mathbb{R}^N}\nu_{X,x}(t,y_1, \ldots y_N , y_{N+1} , \ldots , y_{N+k}) dy_1 \ldots dy_{N}
			\end{align*}
		\end{remark}

		\subsection{Evolution for large, but finite, width}
		
		In this section we shall now estimate how the induced measure $\mathbb{P}_n = (\cF^{(n)})_\ast (Q^{(n)})$ in $C([a,b],\mathbb{R})$ evolves for finite width $n \in \mathbb{N}$ networks. We shall be interested in the regime where $n \gg 1$ and will compare the evolution of $\mathbb{P}_n$ with the evolution of $\mathbb{P}^\infty$, whose finite dimensional distributions have the densities computed in Theorem\ref{prop: density for infinite with along training}.
		
		Let $x \in [a,b]^k$ be fixed and consider the finite dimensional distributions defined in \ref{eq:measure along training for finite n} and \ref{eq:measure along training for infinite n} which, for completeness, we recall here to be
		\begin{align*}
			\mathbb{P}^{(n)}_x(t) & =(\cF_x^{(n)}\circ \phi_t^{(n)})_* Q^{(n)} \\
			\mathbb{P}^{\infty}_x(t) & = (\Phi_t^{\infty})_* \mathbb{P}^{\infty}_x ,
		\end{align*}
		where: (i) the function $\mathcal{F}_x^{(n)}:\mathbb{R}^{4n} \to \mathbb{R}^k$ denotes the composition of $\cF^{(n)}$ with the evaluation at a fixed $x$ as in \ref{eq:Fnx}, i.e. it is the evaluation of the neural network at $x$; (ii) $\phi^{(n)}_t:\mathbb{R}^{4n} \to \mathbb{R}^{4n}$ gives the evolution of the parameters from the finite neural network when trained by gradient descent; (iii) $\Phi_t^{\infty}$ gives the evolution of the outputs when trained using the infinite network limit as in equations \ref{eq:Phi infinity} and \ref{eq:Infinite width network}.
		
		It will prove useful in ``interpolating'' between $\mathbb{P}^{(n)}_x(t)$ and $\mathbb{P}^{(n)}_x(t)$ to introduce the auxiliary measure 
		$$\mathbb{P}^{(n),\infty}(t)= \big(\Phi_t^{\infty} \circ  \cF^{(n)}_x\big)_*  Q^{(n)}$$
		on $\mathbb{R}^k$. This is obtained by first evaluating the finite width network and then evolving the output according to the flow induced by the infinite width network.
		
		In order to better visualize $\mathbb{P}^{(n)}(t)$ and $\mathbb{P}^{(n),\infty}(t)$, we note that these are obtained by pushforward along the two different diagonals of the noncummatitive diagram \ref{fig:diagram} below.
		
		\begin{equation}\label{fig:diagram} 
			\begin{tikzcd}
				\mathbb{R}^{4n} \arrow{r}{\phi^{(n)}_t} \arrow[swap]{d}{\mathcal{F}_x^{(n)}} & \mathbb{R}^{4n} \arrow{d}{\mathcal{F}_x^{(n)}} \\%
				\mathbb{R}^{k} \arrow{r}[swap]{\Phi^\infty_t}& \mathbb{R}^{k}
			\end{tikzcd}
		\end{equation}
		
		Recall that our intention is to measure the distance between $\mathbb{P}^{(n)}_x(t)$ and $\mathbb{P}^{\infty}_x(t)$, and in particular how that distance changes with $n$ as we evolve in $t$. To this purpose, given two measures $\mu_1, \mu_2$, we will consider the Prokhorov metric between them, $\pi(\mu_1, \mu_2)$, defined as the infimum over the set of positive $\rho>0$ that simultaneously satisfies $\mu_1(A) \leq \mu_2(A^\rho)+\rho$ and $\mu_2(A) \leq \mu_1(A^\rho)+\rho$, with $A^\rho=\bigcup_{a \in A} B_\rho(a)$. In formulas,
		\[
		\pi(\mu_1, \mu_2)=\inf\big\{\rho>0: \forall \ {\text{measurable }   A} , \  \mu_1(A) \leq \mu_2(A^\rho)+\rho \mbox{ and } \mu_2(A) \leq \mu_1(A^\rho)+\rho \big\}.
		\]
		Our main result in this section is the following.
		
		\begin{theorem}\label{thm:distance between measures}
			Let $x \in \mathbb{R}^k$. Then, there is a constant $C>0$, independent of $n$, such that for all $t \geq 0$
			$$\pi \Big(  \mathbb{P}^{(n)}_x(t) , \mathbb{P}^{\infty}_x(t)  \Big) \leq C \frac{(\log(n))^{+}}{\sqrt{n}} .$$
		\end{theorem}
		
		To prove this Theorem, we start by proving a series of three intermediate results. In these we shall use the notation
		$$\e_n:=C\frac{(\log(n))^{+}}{\sqrt{n}}.$$
		
		\begin{lemma}
			There is $m \in \mathbb{N}$ and $C'>0$ with the following significance. For $n>m$ there is a set $\mathcal{A}_{\delta_n}^n \subset \mathbb{R}^{4n}$ with $Q^{(n)}(\cA_{\delta_n}^n) \geq 1-C' \frac{(\log(n))^{1/2+}}{\sqrt{n}}$, such that for any $A\subseteq \mathbb{R}^{4n}$, 
			\begin{align}
				&\Bigg(\big( \Phi^\infty_t \circ  \cF^{(n)}_x \big)^{-1}(A) \Bigg) \cap \mathcal{A}_{\delta_n}^n  \;\; \subseteq \;\; \Bigg(\big( \cF^{(n)}_x \circ \phi_t^{(n)} \big)^{-1}(A^{\e_n}) \Bigg)\cap \mathcal{A}_{\delta_n}^n  \label{eqt1_inclusion} 
				\\
				&\Bigg(\big( \cF^{(n)}_x \circ \phi_t^{(n)} \big)^{-1}(A)\Bigg) \cap \mathcal{A}_{\delta_n}^n \;\; \subseteq \;\; \Bigg(\big( \Phi^\infty_t \circ  \cF^{(n)}_x \big)^{-1}(A^{\e_n})\Bigg) \cap \mathcal{A}_{\delta_n}^n \label{eqt2_inclusion} 
			\end{align}
		\end{lemma}
		\begin{proof}
			We will use the result $(3.13)$ of theorem $2$ that states that for any $\delta \in (0,1)$, there are $C_\delta>0$ and a set $\mathcal{A}_\delta^n \subset \mathbb{R}^{4n}$ with $Q^{(n)}(\cA_\delta^n) \geq 1-\delta$, such that for all $\theta_0 \in \mathcal{A}_\delta^n$
			$$
			\big|\cF^{(n)}_x \circ \phi_t^{(n)}(\theta_0)- \Phi^\infty_t \circ  \cF^{(n)}_x(\theta_0) \big| \leq C_{\delta} \frac{(\log(n))^{1/2+}}{\sqrt{n}}.
			$$
			On the other hand, we can let $\delta_n \sim \frac{(\log(n))^{1/2+}}{\sqrt{n}} \to 0$ as in example \ref{ex:sequences delta and C}. Then, as shown in there 
			$$C_{\delta_n} \frac{(\log(n))^{1/2+}}{\sqrt{n}} = \e_n$$ 
			and we find that there is $\mathcal{A}_{\delta_n}^n$ with $Q^{(n)}(\cA_{\delta_n}^n) \geq 1-\delta_n$ such that for all $\theta_0 \in \mathcal{A}_{\delta_n}^n$ 
			\begin{equation}\label{boundd_evo}
				\big|\cF^{(n)}_x \circ \phi_t^{(n)}(\theta_0)- \Phi^\infty_t \circ  \cF^{(n)}_x(\theta_0) \big| \leq \e_n.
			\end{equation}
			Next, we prove that for any $y \in \mathbb{R}^k$ we have that
			\begin{align*} 
				& \Bigg(\big( \Phi^\infty_t \circ  \cF^{(n)}_x \big)^{-1}(y)\Bigg)\cap \mathcal{A}_{\delta_n}^n   \;\; \subseteq \;\; \Bigg(\big( \cF^{(n)}_x \circ \phi_t^{(n)} \big)^{-1}(B_{\e_n}(y))\Bigg) \cap \mathcal{A}_{\delta_n}^n   
				\\
				& \Bigg(\big( \cF^{(n)}_x \circ \phi_t^{(n)} \big)^{-1}(y)\Bigg) \cap \mathcal{A}_{\delta_n}^n   \;\; \subseteq \;\; \Bigg(\big( \Phi^\infty_t \circ  \cF^{(n)}_x \big)^{-1}(B_{\e_n}(y))\Bigg) \cap \mathcal{A}_{\delta_n}^n 
			\end{align*}
			To prove the first inclusion, pick one $\theta_0 \in \mathcal{A}_{\delta_n}^n$ such that $\cF^{(n)}_x \circ \phi_t^{(n)} (\theta_0)=y$. From \eqref{boundd_evo}, we have that 
			\[
			\big|y- ( \Phi^\infty_t \circ  \cF^{(n)}_x )(\theta_0)\big|\leq \e_n.
			\] 
			That is $\Phi^\infty_t \circ  \cF^{(n)}_x (\theta_0) \in B_{\e_n}(y)$ and the inclusion follows. The second inclusion is proved along the same lines. 
			Finally, the stated result follows from taking the union of these inclusions over all $y \in A$. 
		\end{proof}
		
		Next we use this result to compute a $t$-independent upper bound on the Prokhorov distance between $\mathbb{P}^{(n)}_x(t)$ and the auxiliary measure $\mathbb{P}^{(n),\infty}_x(t)$
		
		\begin{lemma}\label{distance_in_training}
			There is $m>0$, such that for $n >m$
			\[
			\pi\big( \mathbb{P}^{(n),\infty}_x(t) , \mathbb{P}^{(n)}_x(t) \big)\leq \e_n .
			\]
		\end{lemma}
		\begin{proof}
			It follows from \eqref{eqt1_inclusion} that
			\[
			Q^{(n)} \bigg(\big( \cF^{(n)}_x \circ \phi_t^{(n)}  \big)^{-1}(A) \cap \mathcal{A}_{\delta_n}^n \bigg) 
			\leq Q^{(n)} \bigg(\big( \Phi^\infty_t \circ  \cF^{(n)}_x  \big)^{-1}(A^{\e_n})\cap \mathcal{A}_{\delta_n}^n\bigg),
			\]
			since 
			\begin{align*}
				\big(  \mathbb{P}^{(n)}_x(t) \big)(A)=& Q^{(n)} \bigg(\big( \cF^{(n)}_x \circ \phi_t^{(n)} \big)^{-1}(A) \cap \mathcal{A}_{\delta_n}^n\bigg)+ Q^{(n)} \bigg(\big( \cF^{(n)}_x \circ \phi_t^{(n)}  \big)^{-1}(A) \cap \big(\mathcal{A}_{\delta_n}^n\big)^c\bigg)
				\\
				&\leq Q^{(n)} \bigg(\big( \Phi^\infty_t \circ  \cF^{(n)}_x  \big)^{-1}(A^{\e_n})\cap \mathcal{A}_{\delta_n}^n\bigg) + Q^{(n)} \Big( \big(\mathcal{A}_{\delta_n}^n\big)^c\Big)
				\\
				&\leq Q^{(n)} \bigg(\big( \Phi^\infty_t \circ \cF^{(n)}_x  \big)^{-1}(A^{\e_n})\bigg) + \delta_n
				\\
				&=\big( \mathbb{P}^{(n),\infty}_x(t) \big) +\delta_n \\
				&\leq \big( \mathbb{P}^{(n),\infty}_x(t) \big)(A^{\e_n})+ \e_n,
			\end{align*}
			as $\delta >0$ was arbitrary, we find that actually
			$$ \big(  \mathbb{P}^{(n)}_x(t) \big)(A) \leq \big( \mathbb{P}^{(n),\infty}_x(t) \big) (A^{\e_n})+ \e_n.$$
			Furthermore, by following a similar analysis, now using \ref{eqt2_inclusion} instead of \ref{eqt1_inclusion}, we get that $\big( \mathbb{P}^{(n),\infty}_x(t) \big) (A)\leq \big( \mathbb{P}^{(n)}_x(t) \big) (A^{\e_n})+\e_n$, thus proving the result stated in the lemma. 
		\end{proof}
		
		Furthermore, we can also compare $\mathbb{P}_x^\infty(t)$ and $\mathbb{P}^{(n),\infty}_x(t)$.
		
		\begin{lemma}\label{distance_init_inftrain}
			There is $m>0$ and $C''>0$, such that for $n >m$
			$$
			\pi\big( \mathbb{P}^{(n),\infty}_x(t), \mathbb{P}_x^\infty(t) \big)\leq \frac{C''}{\sqrt{n}}  
			$$
		\end{lemma}
		\begin{proof}
			using theorem $1$ in the form that says us that the density $\nu_x^{n}$ of $\mathbb{P}_x^{(n)}(0)=\mathbb{P}_x^{(n)}$,  is close to the density $\nu_x^{\infty}$ of the $\mathbb{P}_x^\infty(0)=\mathbb{P}_x^\infty$, and in particular it says 
			\[
			\big|\nu^{n}(y) - \nu^{\infty}(y)\big|\leq \frac{\bar C}{\sqrt{n}} \lambda(y),
			\]
			for some positive function $\lambda$ in $\mathbb{R}^k$ with finite integral. We then have, for sufficiently large $n \gg 1$, that 
			\begin{align*}
				\Big( \mathbb{P}^{(n),\infty}_x(t) \Big) (A) & = \mathbb{P}_x^{(n)}\Big(\Phi_{-t}^\infty (A) \Big) \\
				&\leq \mathbb{P}_x^{\infty}\Big(\Phi_{-t}^\infty(A) \Big)+\frac{\bar C}{\sqrt{n}} \Lambda\big(\Phi_{-t}^\infty(A)\big) ,
			\end{align*}
			where $\Lambda$ denotes the finite measure in $\mathbb{R}^{k}$ with density $\nu$. We conclude that 
			$$\Big( \mathbb{P}^{(n),\infty}_x(t) \Big) (A) \leq  \Big( \mathbb{P}_x^{\infty}(t) \Big) (A^{\frac{C''}{\sqrt{n}}  })+{\frac{C''}{\sqrt{n}}  }.$$
			The same reasoning leads us to the inequality 
			\begin{align*}
				\Big( \mathbb{P}_x^{\infty}(t) \Big) (A) & =  \mathbb{P}_x^{\infty} \Big(\Phi_{-t}^\infty(A) \Big) \\
				& = \mathbb{P}_x^{(n)}\Big(\Phi_{-t}^\infty(A) \Big)+ \frac{\bar{C}\Lambda\big(\Phi_{-t}^\infty(A)\big)}{\sqrt{n}}
				\\
				& \leq \Big( \mathbb{P}^{(n),\infty}_x(t) \Big) (A)+ \frac{C''}{\sqrt{n}}   \\
				& \leq  \Big( \mathbb{P}^{(n),\infty}_x(t) \Big) (A^{\frac{C''}{\sqrt{n}}  }) + \frac{C''}{\sqrt{n}}   ,
			\end{align*}
			which finalizes the proof of the stated result.
		\end{proof}
		
		\begin{proof}[Proof of Theorem \ref{thm:distance between measures}]
			From lemmas \ref{distance_in_training} and \ref{distance_init_inftrain} 
			$$
			\pi\big( \mathbb{P}^{(n)}_x(t), \mathbb{P}_x^\infty(t) \big) \leq \pi\big( \mathbb{P}^{(n),\infty}_x(t), \mathbb{P}_x^{(n)}(t) \big) + \pi\big( \mathbb{P}^{(n),\infty}_x(t), \mathbb{P}_x^\infty(t) \big) \leq 2 \e_n .
			$$
		\end{proof}

		\appendix

		\section{Edgeworth Expansion}
		In the present Appendix we review a well known method in probability theory known as Edgeworth series expansion.
		See for example \cite{Hall} and references therein for a comprehensive exposition. 
		

		\subsection{Characteristic functions}\label{sec:Characteristic functions}
		
		Let $X=(X_1, \ldots , X_k)$ be a $\mathbb{R}^k$-valued random variable all of whose moments are finite. We shall denote by $f_X(y)=f_X(y_1, \ldots , y_k)$ its density function and by
		$$Z_X(p_1, \ldots , p_k) = \mathbb{E} ( e^{i\sum_{i=1}^kp_iX_i}) ,$$
		the characteristic function of $X$. In particular, writing this explicitly as $\int_{\mathbb{R}^k} f_X(y_1, \ldots ,y_k) e^{i\sum_{i=1}^kp_iy_i} dy_1 \ldots dy_k$ one immediately identifies it with the inverse Fourier transform of $f_X$. Thus, we can recover $f_X$ from $Z_X$ via the Fourier transform
		$$f_X(y_1, \ldots , y_k) = \frac{1}{(2\pi)^k} \int_{\mathbb{R}^k} Z_X(p_1, \ldots, p_k) e^{-i\sum_{i=1}^kp_iy_i} dp_1 \ldots dp_k .$$
		One important property of characteristic function is the following weak convergence result. If $\lbrace Y_n \rbrace_{n=1}^{+\infty}$ is a family of random variables whose characteristic functions $Z_{Y_n}$ pointwise converging to a function $Z_{Y}$ continuous at the origin with $Z_{Y}(0)=1$, then there is a random variable $Y$ with characteristic function is $Z_Y$ and 
		$$f_{Y_n}(y) \ dy_1 \ldots dy_k \rightharpoonup f_Y(y) \ dy_1 \ldots dy_k  ,$$ 
		i.e. the $Y_n$ converge to $Y$ weakly.
		
		A few further useful properties of characteristic functions can now be easily proven directly from the definition. 
		
		\begin{itemize}
			\item For $r \in \mathbb{N}^k$ we have
			\begin{align*}
				Z_X^{(r_1, \ldots ,r_k )}(0) & := \frac{d^{r_1} \ldots d^{r_k} }{dp_1^{r_1} \ldots dp_k^{r_k}} \Big\vert_{p=0} Z_X = \mathbb{E} ( (iX_1)^{r_1} \ldots (iX_k)^{r_k} e^{i\sum_{i=1}^k p_iX_i} |_{p=0} ) \\
				& = i^{r_1 + \ldots + r_k} \mathbb{E} ( X_1^{r_1} \ldots X_k^{r_k} ) = i^{r_1 + \ldots +r_k} \mu_{r_1, \ldots r_k},
			\end{align*}
			where $\mu(r_1, \ldots , r_k)$ is the $r_1, \ldots , r_k$-th moment of $X$ if this is finite. As a special application of this result and the Taylor formula we find that if the first $l$ moments of $X$ are all finite, then
			\begin{equation}\label{eq:Characteristic fn}
				Z_X(p) = 1 + \sum_{r_1+ \ldots + r_k=1}^l \mu_{r_1, \ldots r_k} \frac{(i p_1)^{r_1} \ldots (i p_k)^{r_k} }{r_1! \ldots r_k!} + R_l(p),
			\end{equation}
			with the remainder term $R_l$ satisfying $\lim_{q\to 0} \frac{R_l(q)}{|q|^l}=0$. 
			
			\item If $X,\tilde{X}$ are independent random variables as before and $a,b \in \mathbb{R}$, then
			$$Z_{a X + b \tilde{X}}(p) = Z_{X}(a p)Z_{\tilde{X}}(b p).$$
		\end{itemize}

		\subsection{Cumulant generating function}
		
		It follows from equation \ref{eq:Characteristic fn} that the characteristic function $Z_X$ can be interpreted as the moment generating function. In the same way the function $F_X:=\log Z_X$ admits an expansion
		\begin{equation}\label{eq:Cumulants fn}
			F_X(p) = \sum_{r_1+ \ldots + r_k=1}^l \lambda_{r_1, \ldots r_k} \frac{(i p_1)^{r_1} \ldots (i p_k)^{r_k} }{r_1! \ldots r_k!} + \tilde{R}_l(p),
		\end{equation}
		with the remainder term $\tilde{R}_l$ satisfying $\lim_{q\to 0} \frac{\tilde{R}_l(q)}{|q|^l}=0$. Here, the quantities $\lambda_{r_1, \ldots r_k}$ are known as the cumulants. These are related with the moments by equation \ref{eq:Cumulants fn} with the logarithm of \ref{eq:Characteristic fn} as follows
		\begin{align}\label{eq:cumulants from moments}
			\sum_{r_1+ \ldots + r_k=1}^\infty \lambda_{r_1, \ldots r_k} \frac{(i p_1)^{r_1} \ldots (i p_k)^{r_k} }{r_1! \ldots r_k!} = \sum_{j=1}^\infty \frac{(-1)^{j+1}}{j} \left( \sum_{s_1+ \ldots + s_k=1}^\infty \mu_{s_1, \ldots s_k} \frac{(i p_1)^{s_1} \ldots (i p_k)^{s_k} }{s_1! \ldots s_k!} \right)^j ,
		\end{align}
		which allows for writing the cumulants in terms of the moments.
		
		\begin{remark}		
			Expanding the first few terms of \ref{eq:cumulants from moments} we find that
			\begin{align*}
				\frac{\lambda_{r_1, \ldots , r_k}}{r_1! \ldots r_k!} & = \frac{\mu_{r_1, \ldots , r_k}}{r_1! \ldots r_k!} - \frac{1}{2} \sum_{s_i+s_i'=r_i} \frac{\mu_{r_1, \ldots , r_k}}{r_1! \ldots r_k!}  \frac{\mu_{r_1', \ldots , r_k'}}{r_1'! \ldots r_k'!} + \frac{1}{3} \sum_{s_i+s_i'+s_i''=r_i} \frac{\mu_{s_1, \ldots , s_k}}{s_1! \ldots s_k!}  \frac{\mu_{s_1', \ldots , s_k'}}{s_1'! \ldots s_k'!}  \frac{\mu_{s_1'', \ldots , s_k''}}{s_1''! \ldots s_k''!} \\
				& \ \ \ \ + \ldots ,
			\end{align*}
			with the $\ldots$ representing terms which have quartic powers or higher of the moments. 
		\end{remark}
		
		Alternatively, one may instead equate the exponential of \ref{eq:Cumulants fn} with \ref{eq:Characteristic fn}. This gives
		\begin{align*}
			\sum_{r_1+ \ldots + r_k=0}^\infty \mu_{r_1, \ldots r_k} \frac{(i p_1)^{r_1} \ldots (i p_k)^{r_k} }{r_1! \ldots r_k!} & = \exp \left( \sum_{s_1+ \ldots + s_k=1}^\infty \lambda_{s_1, \ldots ,s_k} \frac{(i p_1)^{s_1} \ldots (i p_k)^{s_k} }{s_1! \ldots s_k!} \right) \\
			& = \prod_{s_1+ \ldots + s_k=1}^\infty 	\exp \left( \lambda_{s_1, \ldots ,s_k} \frac{(i p_1)^{s_1} \ldots (i p_k)^{s_k} }{s_1! \ldots s_k!} \right) \\
			& =  \prod_{s_1+ \ldots + s_k=1}^\infty \sum_{m_{s}=0}^\infty \frac{1}{m_{s}!} \left( \lambda_{s_1, \ldots ,s_k} \frac{(i p_1)^{s_1} \ldots (i p_k)^{s_k} }{s_1! \ldots s_k!} \right)^{m_{s}} \\
			& = \sum_{r_1+ \ldots + r_k=0}^\infty \sum_{m_{s}} \prod_{s_1, \ldots , s_k}  \frac{\lambda_{s_1, \ldots ,s_k}^{m_{s}}}{m_{s}!}   \frac{(i p_1)^{r_1} \ldots (i p_k)^{r_k} }{(s_1!)^{m_{s}}  \ldots (s_k!)^{m_{s}} } ,
		\end{align*}
		with the combination of products and sum in $\sum_{m_{s_1}}\prod_{s_1, \ldots , s_k}$ restricted so that $\sum_{s} m_{s} s_i = r_i$ for all $i=1, \ldots , k$. From this we find that
		\begin{align}\label{eq:moments from cumulants}
			\mu_{r_1, \ldots r_k} = \sum_{m_{s}} \prod_{s_1, \ldots , s_k} \frac{1}{m_{s}!}   \frac{ (r_1!) \ldots (r_k)! }{(s_1!)^{m_{s}}  \ldots (s_k!)^{m_{s}} } \lambda_{s_1, \ldots ,s_k}^{m_{s}}.
		\end{align}
		
		\begin{remark}
			There is a nice pictorial interpretation of this formula. Consider $r_1$ balls of color $1$, $r_2$ balls of color $2$, up to $r_k$ balls of color $k$. Then, the combinatorial factor above counts the number of dividing these balls into $m_s$ boxes each with $s_i$ balls of color $i$ for each $i=1, \ldots , k$. We may then compute a moment from the cumulants simply by dividing them into boxes in all possible ways with each way of dividing them contributing the product of the cumulants associated with each box.
		\end{remark}
		
		\begin{remark}[$r_1 + \ldots + r_k\leq 2$]\label{rem:r_1 up to r_k below 2}
			When $r_1 + \ldots + r_k=1$ only one $r_i$ can be nonzero in which case it is one, and so also $m_s=1$ always and $s_i=1$. Together with the formula \ref{eq:moments from cumulants} this shows that
			$$\mu_{r_1, \ldots , r_k} = \lambda_{r_1, \ldots , r_k},$$
			whenever $r_1 + \ldots + r_k=1$.
			
			Furthermore, when $r_1 + \ldots + r_k=2$ the values of $\mu_{r_1, \ldots , r_k}$ encode the covariance matrix. Indeed, if $r_i$ and $r_j$ are the only nonzero terms with $r_i+r_j=2$, we have
			$$\mu_{r_1, \ldots r_k} = \mathbb{E}[X_i X_j]=:C_{ij},$$
			and so
			$$\sum_{r_1 + \ldots + r_k =2} \mu_{r_1, \ldots , r_k} (ip_1)^{r_1} \ldots (ip_k)^{r_k} = - \frac{1}{2} \sum_{i,j=1}^k C_{ij} \ p_i p_j.$$
			A special case is when for $r_1 + \ldots + r_k=1$ all $\mu_{r_1, \ldots , r_k} =0 = \lambda_{r_1, \ldots , r_k}$. In this situation we also have $\mu_{r_1, \ldots , r_k} = \lambda_{r_1, \ldots , r_k}$ whenever $r_1 + \ldots + r_k=2$ and so 
			$$\sum_{r_1 + \ldots + r_k =2} \lambda_{r_1, \ldots , r_k} (ip_1)^{r_1} \ldots (ip_k)^{r_k} = - \frac{1}{2} \sum_{i,j=1}^k C_{ij} \ p_i p_j,$$
			as well.
		\end{remark}
		
		Let us now see a particular example.
		
		\begin{example}[$k=1$]\label{ex:cumulants and moments k=1}
			In this particular situation, our formula \ref{eq:moments from cumulants} for writing the moments from the cumulants says that
			\begin{align}
				\mu_{r} = \sum_{m_s} \prod_{s} \frac{1}{m_{s}!}   \frac{ r!}{(s!)^{m_{s}} } \lambda_{s}^{m_{s}},
			\end{align}
			with the constraint that $\sum_s s m_s = r$. In particular cases this gives
			\begin{align*}
				\mu_1 & = \lambda_1 \\
				\mu_2 & = \lambda_2 + \lambda_1^2 \\
				\mu_3 & = \lambda_3 + 3 \lambda_1 \lambda_2 + \lambda_1^3 \\
				\mu_4 & = \lambda_4 + 4 \lambda_3 \lambda_1 + 3 \lambda_2^2 + 6 \lambda_2 \lambda_1^2 + \lambda_1^4 ,
			\end{align*}
			and so on. In the opposite direction a more cumbersome but straightforward computation using equation \ref{eq:cumulants from moments} yields
			\begin{align*}
				\lambda_1 & = \mu_1 \\
				\lambda_2 & = \mu_2 - \mu_1^2 \\
				\lambda_3 & = \mu_3 - 3 \mu_1 \mu_2 + 2\mu_1^3 \\
				\lambda_4 & = \mu_4 - 4 \mu_3 \mu_1 - 3 \mu_2^2 + 12 \mu_2 \mu_1^2 - 6 \mu_1^4 ,
			\end{align*}
			and so on. A particularly relevant situation which we will encounter is when $\mu_1=0=\lambda_1$ in which case these last relations get simplified to
			\begin{align*}
				\lambda_2 & = \mu_2  \\
				\lambda_3 & = \mu_3  \\
				\lambda_4 & = \mu_4 - 3 \mu_2^2  .
			\end{align*}
		\end{example}

		\section{Asymptotic refinement of the central limit theorem (Edgeworth expansion)}
		\label{appEdgeworth}

		\subsection{The central limit theorem}
		
		Now, suppose that $\lbrace X^{(n)} \rbrace_{n \in \mathbb{N}}$ is a family of i.i.d. random variables whose first $l$-moments are finite (for $l >2$) and $\mu_{r_1,  \ldots , r_k}=0$ anytime $r_1 + \ldots + r_k=1$, i.e. all first moments vanish. Define,
		$$Y_n := \frac{1}{\sqrt{n}} \sum_{i=1}^n X^{(i)}.$$
		Then, 
		\begin{align}\label{eq:Z_Y_n}
			Z_{Y_n}(p) & = Z_X \left( p/\sqrt{n} \right)^n \\ \nonumber
			& = \left(  1 + \sum_{r_1+ \ldots + r_k=1}^l \frac{\mu_{r_1, \ldots r_k}}{n^{\frac{r_1 + \ldots + r_k}{2}}} \frac{(i p_1)^{r_1} \ldots (i p_k)^{r_k} }{r_1! \ldots  r_k!} + R_l(p/\sqrt{n}) \right)^n \\ \nonumber
			& = \left( 1 + \frac{1}{n} \left(\sum_{r_1+ \ldots + r_k=2}^l \frac{\mu_{r_1, \ldots r_k}}{n^{\frac{r_1 + \ldots + r_k}{2} -1}} \frac{(i p_1)^{r_1} \ldots (i p_k)^{r_k} }{r_1! \ldots r_k!} + nR_l(p/\sqrt{n}) \right) \right)^n.
		\end{align}
		From this formula it is clear that  for any finite $p \in \mathbb{R}$
		$$\lim_{n \to +\infty}Z_{Y_n}(p) = \exp \left( - \sum_{r_1+ \ldots + r_k=2} \mu_{r_1, \ldots r_k} \frac{p_1^{r_1} \ldots p_k^{r_k} }{r_1! \ldots r_k!}  \right),$$
		thus proving the central limit theorem.

		\subsection{Asymptotic behavior in $n \gg 1$ (Proof of Theorem \ref{thm: Asymptotic})}\label{ap:proof of theorem}
		
		However, if instead of attempting to simply compute $\lim_{n \to +\infty}Z_{Y_n}(p)$ we want to evaluate the rate at which $Z_{Y_n}(p)$ converges to its limit it is convenient to use instead the cumulant generating function. The main reason for this can already be seen in equation \ref{eq:Z_Y_n} as taking its logarithm will bring the $n$ in the exponent down. Indeed, we find that
		\begin{align*}
			F_{Y_n}(p) & := \log Z_{Y_n}(p) \\
			& = n F_X(p/\sqrt{n}) \\
			& = n \sum_{r_1+ \ldots + r_k=1}^l \frac{\lambda_{r_1, \ldots r_k}}{n^{\frac{r_1 + \ldots + r_k}{2}}} \frac{(i p_1)^{r_1} \ldots (i p_k)^{r_k} }{r_1! \ldots r_k!} + n\tilde{R}_l(p/\sqrt{n})
		\end{align*}
		Now, we set $\rho \gg 1$ to be fixed later and consider $p$ such that $|p|<\rho$. Then, there are positive constants $n_\rho,C_\rho>0$ such that for all $n > n_\rho$ $p$ and $|p|\leq \rho$ we have
		$$n |\tilde{R}_l(p/\sqrt{n})| \sim o \left( \frac{C_\rho}{n^{l/2-1}} \right).$$ 
		At this point we make use of the assumption that $\lambda_{r_1, \ldots , r_k}=0 = \mu_{r_1 , \ldots , r_k}$ when $r_1+ \ldots + r_k=1$. Inserting above this leads to the conclusion that for $|p| \leq \rho$
		\begin{align*}
			F_{Y_n}(p) & =  - \sum_{r_1 + \ldots + r_k =2}   \frac{\lambda_{r_1, \ldots , r_k}}{r_1! \ldots r_k!} (i p_1)^{r_1} \ldots (i p_k)^{r_k}  \\
			& \ \ \ \ \ +  \sum_{r_1 + \ldots + r_k >2}^l \frac{1}{n^{\frac{r_1+ \ldots + r_k}{2} - 1 }}  \frac{\lambda_{r_1, \ldots , r_k}}{r_1! \ldots r_k!} (i p_1)^{r_1} \ldots (i p_k)^{r_k}  + o(n^{-l/2+1})
		\end{align*}
		with the $o(n^{-l/2+1})$ depending only on $\rho$. Hence, for such $p$ with $|p| \leq \rho$
		\begin{align*}
			Z_{Y_n}(p) & = \exp F_{Y_n} (p) \\
			& = \exp \left(- \sum_{r_1 + \ldots + r_k =2}   \frac{\lambda_{r_1, \ldots , r_k}}{r_1! \ldots r_k!} (i p_1)^{r_1} \ldots (i p_k)^{r_k} \right)  \times \\
			& \ \ \ \ \ \times \exp \left( \sum_{r_1 + \ldots + r_k >2}^l \frac{1}{n^{\frac{r_1+ \ldots + r_k}{2} - 1 }}  \frac{\lambda_{r_1, \ldots , r_k}}{r_1! \ldots r_k!} (i p_1)^{r_1} \ldots (i p_k)^{r_k}  + o(n^{-l/2+1}) \right) .
		\end{align*}
		Bearing in mind the discussion in remark \ref{rem:r_1 up to r_k below 2}, the first exponential is given by
		$$\exp \left(-\frac{1}{2} \sum_{i,j=1}^k  C_{ij} \ p_i p_j \right),$$
		which we shall write as $e^{- \frac{1}{2} C_{ij}  p_i p_j}$ using Einstein summation convention. Furthermore, expanding the second exponential we find that it can be written as
		\begin{align*}
			1 + \sum_{m=1}^l \frac{1}{m!} \left( \sum_{r_1 + \ldots + r_k >2}^l \frac{1}{n^{\frac{r_1+ \ldots + r_k}{2} - 1 }}  \frac{\lambda_{r_1, \ldots , r_k}}{r_1! \ldots r_k!} (i p_1)^{r_1} \ldots (i p_k)^{r_k}  \right)^m  + o(n^{-l/2+1}).
		\end{align*}
		We then define
		\begin{align*}
			Z_{Y_n}^l(p) = \left( 1 + \sum_{m=1}^l \frac{1}{m!} \left( \sum_{r_1 + \ldots + r_k >2}^l \frac{1}{n^{\frac{r_1+ \ldots + r_k}{2} - 1 }}  \frac{\lambda_{r_1, \ldots , r_k}}{r_1! \ldots r_k!} (i p_1)^{r_1} \ldots (i p_k)^{r_k} \right)^m  \right) e^{-\frac{1}{2} \sum_{i,j=1}^k  C_{ij} p_i p_j },
		\end{align*}
		which has the property that for $|p| \leq \rho$
		\begin{equation}\label{eq: Appendix Compare Z and Z^l}
			|Z_{Y_n}(p) - Z_{Y_n}^l(p)| \lesssim o \left( n^{-l/2+1} |p|^l e^{-\frac{1}{2} \sum_{i,j=1}^k  C_{ij} p_i p_j } \right).
		\end{equation}
		
		Recall that our ultimate goal is to compute the density function of $Y_n$ which is given by the Fourier transform of $Z_{Y_n}(p)$. As we shall now see, up to terms of order $o(n^{-l/2+1})$ this is given by the Fourier transform of $Z_{Y_n}^l(p)$, as follows
		\begin{align*}
			f_{Y_n} (Y) & = \frac{1}{(2\pi)^k} \int_{\mathbb{R}^k} Z_{Y_n} (p_1 , \ldots , p_k) e^{-ip_iy_i} dp_1 \ldots dp_k \\
			& = \frac{1}{(2\pi)^k} \int_{|p| \leq \rho} Z_{Y_n} (p_1 , \ldots , p_k) e^{-ip_iy_i} dp_1 \ldots dp_k + \frac{1}{(2\pi)^k} \int_{|p| \geq \rho} Z_{Y_n} (p_1 , \ldots , p_k) e^{-ip_iy_i} dp_1 \ldots dp_k \\
			& = \frac{1}{(2\pi)^k} \int_{|p| \leq \rho} Z_{Y_n} (p_1 , \ldots , p_k) e^{-ip_iy_i} dp_1 \ldots dp_k + \epsilon_\rho ,
		\end{align*}
		with $\lim_{\rho \to + \infty} \epsilon_\rho=0$. Then, using equation \ref{eq: Appendix Compare Z and Z^l}
		\begin{align*}
			f_{Y_n} (y) & = \frac{1}{(2\pi)^k} \int_{|p| \leq \rho} Z^l_{Y_n} (p_1 , \ldots , p_k) e^{-ip_iy_i} dp_1 \ldots dp_k \\
			& \ \ \ \  + \frac{1}{(2\pi)^k} \int_{|p| \leq \rho} o \left( n^{-l/2+1} |p|^l e^{-\frac{1}{2} \sum_{i,j=1}^k  C_{ij} p_i p_j } \right) e^{-ip_iy_i} dp_1 \ldots dp_k + \epsilon(\rho) \\
			& = \frac{1}{(2\pi)^k} \int_{\mathbb{R}^k} Z^l_{Y_n} (p_1 , \ldots , p_k) e^{-ip_iy_i} dp_1 \ldots dp_k \\
			& \ \ \ \  + \delta_\rho + \frac{1}{(2\pi)^k} \int_{|p| \leq \rho} o \left( n^{-l/2+1} |p|^l e^{-\frac{1}{2} \sum_{i,j=1}^k  C_{ij} p_i p_j } \right) e^{-ip_iy_i} dp_1 \ldots dp_k + \epsilon_\rho ,
		\end{align*}	
		where 
		$$\delta_\rho = - \frac{1}{(2\pi)^k} \int_{|p| > \rho} Z^l_{Y_n} (p_1 , \ldots , p_k) e^{-ip_ix_i} dp_1 \ldots dp_k,$$
		which also satisfies $\lim_{\rho \to + \infty} \delta_\rho=0$. Then, we can pick a sequence $\lbrace \rho_n \rbrace_{n \gg 1}$ converging to infinity sufficiently fast so that $\epsilon_{\rho_n} + \delta_{\rho_n} = o(n^{-l/2+1})$
		\begin{align}\nonumber
			f_{Y_n} (y) & =  \frac{1}{(2\pi)^{k}} \int_{\mathbb{R}^k} \left( 1 + \sum_{m=1}^l \frac{1}{m!} \left( \sum_{r_1 + \ldots + r_k >2}^l \frac{1}{n^{\frac{r_1+ \ldots + r_k}{2} - 1 }}  \frac{\lambda_{r_1, \ldots , r_k}}{r_1! \ldots r_k!} (i p_1)^{r_1} \ldots (i p_k)^{r_k} \right)^m \right) \times \\ \nonumber
			& \ \ \ \  \times e^{- \frac{1}{2}C_{ij} p_i p_j} e^{-ip_iy_i} dp_1 \ldots dp_k  + o(n^{-l/2+1}) \\ \label{eq: Appendix f Y_n 1}
			& =  \left( 1 + \sum_{m=1}^l \frac{1}{m!} \left( \sum_{r_1 + \ldots + r_k >2}^l \frac{1}{n^{\frac{r_1+ \ldots + r_k}{2} - 1 }}  \frac{\lambda_{r_1, \ldots , r_k}}{r_1! \ldots r_k!} (-\partial_{y_1})^{r_1} \ldots (-\partial_{y_k})^{r_k} \right)^m \right) \frac{e^{- \frac{1}{2} C^{-1}_{ij}  y_i y_j}}{ \sqrt{(2\pi)^{k} \det(C)} }  \\ \nonumber
			& \ \ \ \ + o(n^{-l/2+1}) ,
		\end{align}
		where $C^{-1}_{ij}$ denotes the $i,j$ entry of the inverse of the covariance matrix $C_{ij}$. Notice that formally we can actually rewrite this as 
		\begin{equation}\label{eq: Appendix f Y_n 2}
			f_{Y_n} (y) =  \exp \left( \sum_{r_1 + \ldots + r_k >2}^l \frac{1}{n^{\frac{r_1+ \ldots + r_k}{2} - 1 }}  \frac{\lambda_{r_1, \ldots , r_k}}{r_1! \ldots r_k!} (-\partial_{y_1})^{r_1} \ldots (-\partial_{y_k})^{r_k} \right) \frac{e^{- \frac{1}{2} C^{-1}_{ij}  y_i y_j}}{ \sqrt{(2\pi)^{k} \det(C)} }   ,
		\end{equation}
		up to terms of order $o(n^{-l/2+1})$. In terms of $\mu_{r_1, \ldots , r_k}$ this can also be written by making use of the relations \ref{eq:cumulants from moments} which allows us to write the cumulants from the moments.
		\begin{remark}
			As a way of conveying the use of this result we shall now expand
			$$I:=\sum_{m=1}^l \frac{1}{m!} \left( \sum_{r_1 + \ldots + r_k >2}^l \frac{1}{n^{\frac{r_1+ \ldots + r_k}{2}-1 }}  \frac{\lambda_{r_1, \ldots , r_k}}{r_1! \ldots r_k!}  (-\partial_{y_1})^{r_1} \ldots (-\partial_{y_k})^{r_k} \right)^m ,$$
			in powers of $n^{-1/2}$. Expanding out the first few we find
			\begin{align*}
				I & = \frac{1}{\sqrt{n}} \sum_{r_1 + \ldots + r_k =3} \frac{\lambda_{r_1, \ldots , r_k}}{r_1! \ldots r_k!}  (-\partial_{y_1})^{r_1} \ldots (-\partial_{y_k})^{r_k}  +    \\
				& \ \ \ \ + \frac{1}{n}  \sum_{r_1 + \ldots + r_k =4} \frac{\lambda_{r_1, \ldots , r_k}}{r_1! \ldots r_k!}  (-\partial_{y_1})^{r_1} \ldots (-\partial_{y_k})^{r_k} \\
				& \ \ \ \ + \frac{1}{2n} \sum_{r_1 + \ldots + r_k =3} \sum_{r_1' + \ldots + r_k' =3} \frac{\lambda_{r_1, \ldots , r_k}}{r_1! \ldots r_k!}  \frac{\lambda_{r_1', \ldots , r_k'}}{r_1'! \ldots r_k'!}  (-\partial_{y_1})^{r_1+r_1'} \ldots (-\partial_{y_k})^{r_k+r_k'}  + \ldots ,
			\end{align*}
			where the $\ldots$ denote terms of order $O(n^{-3/2})$. Further recall that the cumulants in this expression can be written from the moments using equation \ref{eq:cumulants from moments}
		\end{remark}
		Our main result, stated as \ref{thm: Asymptotic}, follows immediately from applying the results of this section of the Appendix to the random variables $X^{(i)}=(p_i(x_1), \ldots , p_i(x_k))$ in which case $Y_n = \frac{1}{\sqrt{n}} X^{(i)}$ coincides with $(f_\theta(x_1), \ldots , f_\theta(x_k))$.

		\subsection{The case $k=1$ (Details on example \ref{cor: Asymptotic})}\label{ap:proof of corollary}
		
		In order to illustrate how to use our main theorem we shall now do the simplest relevant particular case, namely when $k=1$ and $\mu_1=0=\lambda_1$. Then, equation \ref{eq: Appendix f Y_n 1} gives
		\begin{align*}
			f_{Y_n} (y) & = \left( 1 + \sum_{m=1}^l \frac{1}{m!} \left( \sum_{r >2}^l \frac{1}{n^{\frac{r}{2} - 1 }}  \frac{\lambda_{r}}{r! } (-\partial_{y})^r \right)^m \right) \frac{e^{- \frac{y^2}{2\lambda_2} }}{(2\pi \lambda_2)^{1/2}} .
		\end{align*}
		To compute this we shall expand $I:=\sum_{m=1}^l \frac{1}{m!} \left( \sum_{r >2}^l \frac{1}{n^{\frac{r}{2} - 1 }}  \frac{\lambda_{r}}{r! } (-\partial_{y})^r \right)^m$, as follows
		\begin{align*}
			I & = 
			\sum_{r >2}^l \frac{1}{n^{\frac{r}{2} - 1 }}  \frac{\lambda_{r}}{r! } (-\partial_{y})^r + \frac{1}{2} \sum_{s,s' >2} \frac{1}{n^{\frac{s+s'}{2} - 2 }}  \frac{\lambda_{s}}{s! }  \frac{\lambda_{s'}}{s'!} (-\partial_{y})^{s+s'} \\
			& \ \ \ \ + \frac{1}{6} \sum_{s,s',s'' >2} \frac{1}{n^{\frac{s+s'+s''}{2} - 3 }}  \frac{\lambda_{s}}{s! }  \frac{\lambda_{s'}}{s'!}  \frac{\lambda_{s''}}{s''!} (-\partial_{y})^{s+s'+s''} + \ldots \\
			& = \sum_{r >2}^l \frac{1}{n^{\frac{r}{2}}} \left( n\frac{\lambda_{r}}{r! }  + \frac{n^2}{2} \sum_{s+s' =r}' \frac{\lambda_{s}}{s! }  \frac{\lambda_{s'}}{s'!}  + \frac{n^3}{6} \sum_{s+s'+s'' =r}' \frac{\lambda_{s}}{s! }  \frac{\lambda_{s'}}{s'!}  \frac{\lambda_{s''}}{s''!} + \ldots  \right) \partial^{r}_{y} ,
		\end{align*}
		where the sums $\sum'$ indicate that each terms in the sum must be greater than $2$. 
		
		\begin{remark}
			The last formula can be written in a more systematic manner as
			$$I = \sum_{r >2}^l \frac{1}{n^{\frac{r}{2}}} \left( \sum_{m \geq 1} \frac{n^{m}}{m!} \sum_{s_1 + \ldots + s_m=r} \prod_{i=1}^m \frac{\lambda_{s_i}}{s_i!}\right) (-\partial_{y})^r .$$
		\end{remark}
		
		Expanding the first few terms gives
		\begin{align*}
			I &  = - \frac{\lambda_{3}}{3!n^{1/2}}\partial^{3}_{y}  +  \frac{1}{n}  \frac{\lambda_{4}}{4! } \partial^{4}_{y} -  \frac{1}{n^{3/2}}  \frac{\lambda_{5}}{5! } \partial^{5}_{y} + \frac{1}{n^{3}} \left( n \frac{\lambda_{6}}{6! }  + \frac{n^2}{2} \frac{\lambda_{3}^2}{(3!)^2}  \right) \partial^{6}_{y} + \\
			& \ \ \ \ - \frac{1}{n^{7/2}} \left( n\frac{\lambda_{7}}{7!} + 2\frac{n^2}{2}\frac{\lambda_{3}}{3!}\frac{\lambda_{4}}{4!} \right)\partial^{7}_{y} + \frac{1}{n^{4}} \left( n \frac{\lambda_{8}}{8!} + \frac{n^2}{2} \frac{\lambda_{4}^2}{(4!)^2} \right) \partial^{8}_{y} \\
			& \ \ \ \ -  \frac{1}{n^{9/2}} \left( n  \frac{\lambda_{9}}{9!} + 2\frac{n^2}{2} \frac{\lambda_{5}}{5!} \frac{\lambda_{4}}{4!} + \frac{n^3}{3}\frac{\lambda_{3}^3}{(3!)^3} \right) \partial^{9}_{y} + \ldots 
		\end{align*}
		and arranging these in terms of powers of $n^{-1/2}$ gives
		\begin{align*}
			I &  = - \frac{\lambda_{3}}{3!n^{1/2}}\partial^{3}_{y}  +  \frac{1}{n} \left(  \frac{\lambda_{4}}{4! } \partial^{4}_{y} + \frac{1}{2} \frac{\lambda_{3}^2}{(3!)^2}  \partial^{6}_{y} \right) -  \frac{1}{n^{3/2}} \left(  \frac{\lambda_{5}}{5! } \partial^{5}_{y} +  \frac{\lambda_{3}}{3!}\frac{\lambda_{4}}{4!}  \partial^{7}_{y} + \frac{1}{3}\frac{\lambda_{3}^3}{(3!)^3} \partial^{9}_{y} \right) + \ldots
		\end{align*}
		with the $\ldots$ denoting terms of order $O(n^{-3/2})$. It follows from example \ref{ex:cumulants and moments k=1} that if $\lambda_1=0=\mu_1$, as is the case in our situation, that $\lambda_2=\mu_2$, $\lambda_3=\mu_3$ and $\lambda_4=\mu_4-3\mu_2^2$
		In this situation, we have
		\begin{align*}
			I & = \frac{1}{\sqrt{n}}  \frac{\mu_{3}}{3!} \frac{d^{3}}{d y^{3}}  +  \frac{1}{n}  \left[ \frac{\mu_{4}-3 \mu_2^2}{4!}   \frac{d^{4}}{d y^{4}} + \frac{1}{2} \frac{\mu_{3}^2}{(3!)^2}  \frac{d^{6}}{d y^{6}}  \right] + \ldots ,
		\end{align*}
		and recall that the $\ldots$ represent terms of order $O(n^{-3/2})$. Furthermore, writing the covariance matrix $C$ as $\mu_2^2$, we have $C^{-1}=\frac{1}{\mu_2^2}$ and using the relation between derivatives of a Gaussian and the Hermite polynomials $H_n(x)$ we have
		$$\frac{d^n}{dy^n} \frac{e^{-\frac{y^2}{2\mu_2}}}{\sqrt{2\pi\mu_2}} = \frac{(-1)^n}{(\sqrt{2\mu_2})^n} H_n \left( \frac{y}{\sqrt{2\mu_2}} \right) \frac{e^{-\frac{y^2}{2\mu_2}}}{\sqrt{2\pi\mu_2}} . $$
		Inserting these two equation in either \ref{eq: Appendix f Y_n 1} or \ref{eq: Appendix f Y_n 2} yields
		\begin{align*}
			f_{Y_n}(x) & =\frac{e^{-\frac{x^2}{2\mu_2}}}{\sqrt{2\pi\mu_2}} + \frac{1}{\sqrt{n}}  \frac{\mu_{3}}{3!} \frac{1}{(\sqrt{2\mu_2})^3} H_3 \left( \frac{x}{\sqrt{2\mu_2}} \right) \frac{e^{-\frac{x^2}{2\mu_2}}}{\sqrt{2\pi\mu_2}} \\
			& \ \ \ \  +  \frac{1}{n} \left(\frac{\mu_{4}-3 \mu_2^2}{4!} \frac{1}{(\sqrt{2\mu_2})^4} H_4 \left( \frac{x}{\sqrt{2\mu_2}} \right) + \frac{1}{2} \frac{\mu_{3}^2}{(3!)^2}  \frac{1}{(\sqrt{2\mu_2})^6} H_6 \left( \frac{x}{\sqrt{2\mu_2}} \right) \right) \frac{e^{-\frac{x^2}{2\mu_2}}}{\sqrt{2\pi\mu_2}} + \ldots ,
		\end{align*}
		which we shall more succinctly write as
		\begin{align*}
			f_{Y_n}(y) & =\frac{e^{-\frac{y^2}{2\mu_2}}}{\sqrt{2\pi\mu_2}} + \frac{1}{\sqrt{n}}  \frac{\mu_{3}}{12 \sqrt{2} \mu_2^{3/2}} H_3 \left( \frac{y}{\sqrt{2\mu_2}} \right) \frac{e^{-\frac{y^2}{2\mu_2}}}{\sqrt{2\pi\mu_2}} \\
			& \ \ \ \  +  \frac{1}{n} \left( \frac{\mu_4-3\mu_2^2}{96\mu_2^2} H_4 \left( \frac{y}{\sqrt{2\mu_2}} \right) + \frac{\mu_{3}^2}{576 \mu_2^3} H_6 \left( \frac{y}{\sqrt{2\mu_2}} \right) \right) \frac{e^{-\frac{y^2}{2\mu_2}}}{\sqrt{2\pi\mu_2}} + \ldots ,
		\end{align*}
		which completes the computation leading to the statement in example \ref{cor: Asymptotic}.

		\section{Evolution and convergence of outputs during training}
		\label{appEvolution}
		
		Consider a finite width one-hidden-layer neural network
		\begin{equation}
			\label{fDef}
			f_\theta(\cdot) =f^{(n)}_\theta(\cdot) =  \frac{1}{\sqrt{n}} \sum_{\mu=1}^n p_\mu (\cdot)\;,
		\end{equation}
		build out of perceptrons
		\begin{equation}
			\label{pDef}
			p_{\mu}(x):=\theta_{\mu}^{(3)} \sigma (\theta_{\mu}^{(1)}x+\theta_{\mu}^{(2)})+\theta_{\mu}^{(4)},
		\end{equation}
		where the translation to the previous notation is provided by
		$$W_{\mu}^{(k)}=\theta_{\mu}^{(2k-1)}\;\;\text{ and }\;\;b_{\mu}^{(k)}=\theta_{\mu}^{(2k)}\;.$$
		We will also write $p(\theta,x)$ when we need to make the dependence on parameters explicit. 
		
		The Neural Tangent Kernel associated to $f_\theta$ is by definition
		$$\mathrm{NTK}_{zw}^{(n)}(\theta):=\sum_{i=1}^{4n} \frac{\partial f_\theta (z)}{\partial \theta_i} \frac{\partial f_\theta (w)}{\partial \theta_i}\;,$$
		with the sum taken over the entire range of parameters $\theta_i$, i.e. over $i=1,\cdots, P=4n$.
		
		We then compute
		$$\partial_t \mathrm{NTK}_{zw}^{(n)}(\theta(t)) = \sum_{i=1}^{4n} \left( \frac{d}{dt} \left( \frac{\partial f_\theta (z)}{\partial \theta_i} \right) \frac{\partial f_\theta (w)}{\partial \theta_i} +  \frac{\partial f_\theta (z)}{\partial \theta_i} \frac{d}{dt} \left( \frac{\partial f_\theta (w)}{\partial \theta_i} \right) \right).$$
		Assuming that training is preformed by gradient descent~\eqref{gradientDescent} with mean squared loss~\eqref{msl} we have
		\begin{align*}
			\frac{d}{dt} \left( \frac{\partial f_\theta (z)}{\partial \theta_i} \right) & =  \sum_{j=1}^P \frac{\partial^2 f_\theta (z)}{\partial \theta_j \partial \theta_i} \dot{\theta_j} \\
			& = - \sum_{j=1}^P \frac{\partial^2 f_\theta }{\partial \theta_j \partial \theta_i}(z) \sum_{l=1}^N \frac{\partial f_\theta (X_l)}{\partial \theta_j} (f_\theta(X_l)-Y_l) \\
			& = - \sum_{l=1}^N \left( \sum_{j=1}^P \frac{\partial^2 f_\theta }{\partial \theta_j \partial \theta_i}(z)  \frac{\partial f_\theta (X_l)}{\partial \theta_j} \right) (f_\theta(X_l)-Y_l) ,
		\end{align*}
		and similarly for $\frac{d}{dt} \left( \frac{\partial f_\theta (w)}{\partial \theta_i} \right)$. Inserting this into the equation for the evolution of the neural tangent kernel gives
		\begin{align*}
			\partial_t \mathrm{NTK}_{zw} & = - \sum_{l=1}^N \left( \sum_{i,j=1}^P \left( \frac{\partial^2 f_\theta }{\partial \theta_j \partial \theta_i}(z) \frac{\partial f_\theta (w)}{\partial \theta_i}  \frac{\partial f_\theta (X_l)}{\partial \theta_j} + \frac{\partial^2 f_\theta }{\partial \theta_j \partial \theta_i}(w)  \frac{\partial f_\theta (z)}{\partial \theta_i} \frac{\partial f_\theta (X_l)}{\partial \theta_j} \right)  (f_\theta(X_l)-Y_l) \right)
		\end{align*}
		Now, using the specific form of$f_\theta$ provided by~\eqref{fDef} and~\eqref{pDef} we find that
		\begin{align*}
			\frac{\partial^2 f_\theta }{\partial \theta_j \partial \theta_i}(z) \frac{\partial f_\theta (w)}{\partial \theta_i}  \frac{\partial f_\theta (X_l)}{\partial \theta_j} & = \frac{1}{n^{3/2}} \sum_{\mu, \nu, \lambda =1}^n \frac{\partial^2 p_\mu }{\partial \theta_j \partial \theta_i}(z) \frac{\partial p_\nu (w)}{\partial \theta_i}  \frac{\partial p_\lambda (X_l)}{\partial \theta_j} \\
			& = \frac{1}{n^{3/2}} \sum_{\mu=1}^n \frac{\partial^2 p_\mu }{\partial \theta_j \partial \theta_i}(z) \frac{\partial p_\mu (w)}{\partial \theta_i}  \frac{\partial p_\mu (X_l)}{\partial \theta_j} .
		\end{align*}
		Then,
		\begin{align*}
			\partial_t \mathrm{NTK}_{zw}^{(n)}(\theta(t))  & = -  \frac{1}{n^{3/2}} \sum_{l=1}^N \left( \sum_{i,j=1}^P \sum_{\mu=1}^n \left( \frac{\partial^2 p_\mu (z) }{\partial \theta_j \partial \theta_i} \frac{\partial p_\mu (w)}{\partial \theta_i}   +   \frac{\partial^2 p_\mu (w)}{\partial \theta_j \partial \theta_i} \frac{\partial p_\mu (z)}{\partial \theta_i}   \right)  \frac{\partial p_\mu (X_l)}{\partial \theta_j}  (f_{\theta(t)}(X_l)-Y_l) \right) 
			\\
			& = -  \frac{1}{n^{3/2}} \sum_{\mu=1}^n  \left( \sum_{i,j=1}^P \sum_{l=1}^N \left( \frac{\partial^2 p_\mu (z) }{\partial \theta_j \partial \theta_i} \frac{\partial p_\mu (w)}{\partial \theta_i}   +   \frac{\partial^2 p_\mu (w)}{\partial \theta_j \partial \theta_i} \frac{\partial p_\mu (z)}{\partial \theta_i}   \right)  \frac{\partial p_\mu (X_l)}{\partial \theta_j}  (f_{\theta(t)}(X_l)-Y_l) \right) 
			\\
			& = -  \frac{1}{\sqrt{n}} \left[\frac{1}{n} \sum_{\mu=1}^n  \left( \sum_{i,j=1}^P \sum_{l=1}^N \left( \frac{\partial^2 p_\mu (z) }{\partial \theta_j \partial \theta_i} \frac{\partial p_\mu (w)}{\partial \theta_i}   +   \frac{\partial^2 p_\mu (w)}{\partial \theta_j \partial \theta_i} \frac{\partial p_\mu (z)}{\partial \theta_i}   \right)  \frac{\partial p_\mu (X_l)}{\partial \theta_j}  (f_{\theta(t)}(X_l)-Y_l) \right) \right] 
			\\
			& = -  \frac{1}{\sqrt{n}} \sum_{l=1}^N (f_{\theta(t)}(X_l)-Y_l) \left[\frac{1}{n} \sum_{\mu=1}^n  \left( \sum_{i,j=1}^P  \left( \frac{\partial^2 p_\mu (z) }{\partial \theta_j \partial \theta_i} \frac{\partial p_\mu (w)}{\partial \theta_i}   +   \frac{\partial^2 p_\mu (w)}{\partial \theta_j \partial \theta_i} \frac{\partial p_\mu (z)}{\partial \theta_i}   \right)  \frac{\partial p_\mu (X_l)}{\partial \theta_j}   \right) \right].
		\end{align*}
		The quantity inside square brackets is the third order kernel 
		$$K^{(3)}_n(z,w,X_l):=\frac{1}{n} \sum_{\mu=1}^n  \left( \sum_{i,j=1}^P  \left( \frac{\partial^2 p_\mu (z) }{\partial \theta_j \partial \theta_i} \frac{\partial p_\mu (w)}{\partial \theta_i}   +   \frac{\partial^2 p_\mu (w)}{\partial \theta_j \partial \theta_i} \frac{\partial p_\mu (z)}{\partial \theta_i}   \right)  \frac{\partial p_\mu (X_l)}{\partial \theta_j}   \right)\;.$$

		\begin{remark}
			As a side note observe that as a consequence of the law of large numbers at initialization (i.e. when $t=0$) $K^{(3)}$ converges a.s. (as $n \to + \infty$) to
			$$K^{(3)}_\infty(z,w,v):=\mathbb{E} \left[ \sum_{i,j} \left( \frac{\partial^2 p (z) }{\partial \theta_j \partial \theta_i} \frac{\partial p (w)}{\partial \theta_i}   +   \frac{\partial^2 p (w)}{\partial \theta_j \partial \theta_i} \frac{\partial p (z)}{\partial \theta_i}   \right)  \frac{\partial p (v)}{\partial \theta_j}   \right]\;,$$
			where $\mathbb{E} \left[ \cdot \right]$ denotes the expected value with respect to the parameters $\lbrace \theta^{(i)} \rbrace_{i =1}^4$ encoding the perceptron $p(\cdot)$. 
		\end{remark}

		From now on fix $n$ sufficiently large, chosen according to a prescription that will be made clear below. If our weights are initialized in an i.i.d. fashion with finite second momenta, it follows from the Central Limit Theorem, as discussed above, and the continuity of $f_{\theta(0)}$ that for all $\delta>0$, there exists $C_{\delta}>0$, independent of all sufficiently large $n$, such that 
		\begin{equation}
			\label{probInit}
			\mathbb{P}\left(|f_{\theta(0)}(x)| \leq C_{\delta} \right) \geq 1-\delta \;\;,\; \forall |x|\leq 1\;.
		\end{equation}
		Our results in this section are restricted to non-polynomial activation functions $\sigma$ which have at most linear growth and bounded first and second order derivatives, i.e.,  we assume there exists $C_{\sigma}>0$ such that
		%
		%
		%
		\begin{equation}
			\label{locBounded}
			|\sigma(z)|\leq C_{\sigma}(|z|+1) \;\;\text{ and }\;\;|\sigma'(z)|+|\sigma''(z)|\leq C_{\sigma}
			\;, \,\forall z\in\mathbb{R}\;.
		\end{equation}
		%
		Introducing the norm
		\begin{equation}
			\|\theta_{\mu}(t)\|_{\infty} = \sup_{j=1,\ldots,4} |\theta_{\mu}^{(j)}(t)|\;, 
		\end{equation}
		%
		%
		%
		we see that, for $|x|\leq 1$, 
		%
		\begin{equation}
			\label{est_dp}
			\left|\frac{\partial p_{\mu}}{\partial \theta_{\mu}^{(i)}}(\theta_{\mu}(t),x)\right|\leq 2 C_{\sigma}  \left(\|\theta_{\mu}(t)\|_{\infty}  +1\right) \;
		\end{equation}
		and also
		\begin{equation}
			\left|\frac{\partial^2 p_{\mu}}{\partial \theta_{\mu}^{(i)}\partial \theta_{\mu}^{(j)}}(\theta_{\mu}(t),x)\right|   \leq  C_{\sigma} \left(\|\theta_{\mu}(t)\|_{\infty} +1\right)^2 \;.
		\end{equation}
		%
		%
		%
		
		As a consequence 
		\begin{align*}
			\left| \sum_{i,j=1}^P  \left( \frac{\partial^2 p_\mu (z) }{\partial \theta_j \partial \theta_i} \frac{\partial p_\mu (w)}{\partial \theta_i}   +   \frac{\partial^2 p_\mu (w)}{\partial \theta_j \partial \theta_i} \frac{\partial p_\mu (z)}{\partial \theta_i}   \right)  \frac{\partial p_\mu (X_l)}{\partial \theta_j} \right| 
			&= 
			\left| \sum_{i,j=1}^4  \left( \frac{\partial^2 p_\mu (z) }{\partial \theta_{\mu}^{(j)} \partial \theta_{\mu}^{(i)}} \frac{\partial p_\mu (w)}{\partial \theta_{\mu}^{(i)}}   +   \frac{\partial^2 p_\mu (w)}{\partial \theta_{\mu}^{(j)} \partial \theta_{\mu}^{(i)}} \frac{\partial p_\mu (z)}{\partial \theta_{\mu}^{(i)}}   \right)  \frac{\partial p_\mu (X_l)}{\partial \theta_{\mu}^{(j)}} \right|
			\\
			&\leq  C C_{\sigma}^3 \left(\|\theta_{\mu}(t)\|_{\infty}  +1\right)^3\;,
		\end{align*}
		from which we conclude that 
		\begin{equation}
			\label{dNTK0}
			\left| \partial_t \mathrm{NTK}_{zw}^{(n)}(\theta(t)) \right|
			\leq 
			\frac1n\sum_{\mu=1}^n C C_{\sigma}^3 \left(\|\theta_{\mu}(t)\|_{\infty}  +1\right)^3 
			\frac{1}{\sqrt{n}}\sum_{l=1}^N\left|f_{\theta(t)}(X_l)-Y_l\right| \;.
		\end{equation}

		To obtain the desired control over all relevant quantities (Kernel, outputs and parameters) uniformly for all $t\geq0$ we proceed via a bootstrap/continuity argument: to do that we consider the set
		\begin{equation}
			{\mathcal T}={\mathcal T}_{\epsilon}:=\{t\geq 0 \,:\, |\mathrm{NTK}_{zw}^{(n)}(\theta(s)) |\leq \epsilon 
			\;,\;\text{for all } s\in[0,t]\;\text{ and all } |z|\;,|w|\leq 1\}\;.
		\end{equation}

		Now, in view of~\eqref{NTKconv2} and the continuity of the kernels on the inputs $(z,w)$, we see that given any $\epsilon>0$ then with probability $1-\delta$ we have that $0\in {\mathcal T}_{\epsilon}$, provided we fixed $n\geq m(\epsilon^{-1})$, where $m$ is a sufficiently large increasing function.  By continuity it is clear that 
		${\mathcal T}_{\epsilon}$ is also closed. Our final goal is to show that for appropriate choices of $n$ and $\epsilon$ the corresponding ${\mathcal T}$ is also open and therefore equal to $[0,+\infty[$. 
		
		For $t\in{\mathcal T}_{\epsilon}$, we are allowed to use~\eqref{dNTK0} and~\eqref{NTKconv2} to establish that, with probability $1-\delta$, 
		\begin{align}
			\nonumber
			|\mathrm{NTK}_{zw}^{(n)}(\theta(t)) -\mathrm{NTK}_{zw}^\infty|
			&\leq 
			|\mathrm{NTK}_{zw}^{(n)}(\theta(0)) -\mathrm{NTK}_{zw}^\infty| 
			+ |\mathrm{NTK}_{zw}^{(n)}(\theta(t)) -\mathrm{NTK}_{zw}^{(n)}(\theta(0)) |
			\\
			\nonumber
			&\leq 
			\frac{C\;(\log n)^{1/2+}}{\sqrt{n}} + 
			\int_0^t  \left| \partial_t \mathrm{NTK}_{zw}^{(n)}(\theta(s))  \right| ds
			\\
			\label{dNTK}
			&\leq 
			\frac{C\;(\log n)^{1/2+}}{\sqrt{n}} + 
			\int_0^t \left[ \frac1n\sum_{\mu=1}^n C C_{\sigma}^3 \left(\|\theta_{\mu}(s)\|_{\infty}  +1\right)^3 \right]
			\frac{1}{\sqrt{n}}\sum_{l=1}^N   \left|f_{\theta(s)}(X_l)-Y_l\right| ds\;.
		\end{align}

		We proceed by estimating the term inside the integral. With that in mind let $K^{\infty}_{ij}=\mathrm{NTK}_{X_iX_j}^{\infty} $,  $K_{ij}=K_{ij}(t)=\mathrm{NTK}_{X_iX_j}^{(n)}(\theta(t))$, $y_i=f_{\theta(t)}(X_i)$ and $E_i = y_i-Y_i$. At this point is important to recall that $K^{\infty}_{ij}$ is a.s. time independent~\eqref{NTKinfty} and is moreover strictly positive definite~\cite{NTK>0}; so by choosing $\epsilon$ sufficiently small we see that the minimum eigenvalue of $K$ is bounded from below by a time independent $\lambda>0$. Then 
		for  $t\in{\mathcal T}_{\epsilon}$  we have (recall~\eqref{eq:Evolution of y})
		$$\partial_t E_j=-K_{ji}E_i\;.$$
		Multiplying by $2E_j$ and summing over $j$ we get 
		\[
		\partial_t \|E\|_2^2=-2E^\intercal K E\leq -2\lambda E^\intercal E\leq2\lambda\|E\|_2^2\;.
		\]
		It follows that $\|E(t)\|_2^2\leq C\|E(0)\|_{2}^2  e^{-2\lambda t}\;$.
		Using that $\|g\|_\infty^2\leq \|g\|_2^2 \leq N\|g\|_\infty^2$ we get
		\[
		\|E(t)\|_\infty^2\leq C\|E(0)\|_2^2e^{-2\lambda t}\leq CN\|E(0)\|_\infty^2e^{-2\lambda t} \;.
		\]
		Thus $\|E(t)\|_\infty\leq C\|E(0)\|_\infty e^{-\lambda t}$. In particular we have
		\begin{equation}
			\label{decayError}
			|f_{\theta(t)}(X_i)-Y_i|= |E_i(t)|=|y_i(t)-Y_i|\leq C\|E(0)\|_{\infty}  e^{-\lambda t}\;, 
		\end{equation}
		%
		%
		%
		for all $t\in{\mathcal T}$, and where $\|E(0)\|_{\infty}:=\sup_{i=1,\ldots,N}|E_i(0)|$. Note that when we finally close our bootstrap argument and show that ${\mathcal T}=[0,\infty[$, the previous estimate will also imply that the training of the finite width network converges to a global minimum.

		To proceed, recall that the evolution of the parameters (of the finite width network) during training is governed by 
		$$\partial_t \theta_{\nu}^{(j)}(t)=- \sum_{l=1}^N \frac{\partial f_\theta (X_l)}{\partial \theta_{\nu}^{(j)}} (f_\theta(X_l)-Y_l)\;.$$
		But 
		$$\frac{\partial f_\theta (x)}{\partial \theta_{\nu}^{(j)}}= \frac{1}{\sqrt{n}}\frac{\partial}{\partial \theta_{\nu}^{(j)}}\sum_{\mu=1}^n p_{\mu}(x)=\frac{1}{\sqrt{n}} \frac{\partial p_{\nu_j}(x)}{\partial \theta_{\nu}^{(j)}}$$
		so that, in view of~\eqref{est_dp}, we have 
		$$\left| \frac{\partial f_{\theta} (x)}{\partial \theta_{\nu}^{(j)}}\right| \leq 2\frac{C_{\sigma}}{\sqrt{n}} \left(\|\theta_\nu\|_{\infty}  +1\right)\;. $$
		Recalling~\eqref{decayError} we see that, for $t\in \mathcal T$,  
		$$\partial_t \theta_{\mu}^{(j)}(t)\leq \left|\partial_t \theta_{\mu}^{(j)}(t)\right|
		\leq \frac{NCC_{\sigma}}{\sqrt{n}} \left(\|\theta_\mu(t)\|_{\infty}  +1\right)\|E(0)\|_{\infty}  e^{-\lambda t} \;.$$
		\begin{align*}
			\left| \theta_{\mu}^{(j)}(t)\right| 
			&\leq
			\left| \theta_{\mu}^{(j)}(0)\right| + \int_0 ^t \left|\partial_t \theta_{\mu}(s)\right| ds
			\\
			& \leq \left | \theta_{\mu}^{(j)}(0)\right| + \int_0^t \frac{NCC_{\sigma}}{\sqrt{n}} \left(\|\theta_\mu(s)\|_{\infty}  +1\right)\|E(0)\|_{\infty}  e^{-\lambda s}ds\;. 
		\end{align*}
		Taking the suprema of both sides of the last estimate gives
		\begin{align*}
			\left\| \theta_{\mu}(t)\right\|_{\infty} +1
			&\leq
			\left\| \theta_{\mu}(0)\right\|_{\infty}  +1
			+ \int_0^t \frac{NCC_{\sigma}}{\sqrt{n}} \left(\|\theta_\mu(s)\|_{\infty}  +1\right)\|E(0)\|_{\infty}  e^{-\lambda s}ds\;,
		\end{align*}
		and applying Gronw\"all's inequality we then get 
		\begin{align*}
			\| \theta_{\mu}(t)\|_{\infty} +1
			&\leq
			\left(\| \theta_{\mu}(0)\|_{\infty}  +1 \right)
			\exp\left(
			\int_0^t \frac{NCC_{\sigma}}{\sqrt{n}} \|E(0)\|_{\infty}  e^{-\lambda s}ds\;
			\right)
			\\
			&\leq 
			\left(\| \theta_{\mu}(0)\|_{\infty}  +1 \right)
			\exp\left(
			\frac{NCC_{\sigma}}{\sqrt{n}\lambda } \|E(0)\|_{\infty}  
			\right)\;.
		\end{align*}
		From~\eqref{probInit}, we see that with probability $1-\delta$ we have
		$$\|E(0)\|_{\infty} \leq C_{\delta}+C_Y\;,$$ 
		for a choice of constant such that $|Y_l|\leq C_Y$, for all $l=1\ldots,N$; in particular $\|E(0)\|_{\infty}$ is bounded by a constant independent of $n$. 
		%
		
		An important consequence of the last estimates is that there exists $C>0$ not depending on $n$ such that  
		\begin{align}
			\frac1n\sum_{\mu=1}^n \left(\|\theta_{\mu}(t)\|_{\infty}  +1\right)^3
			&\leq 
			\frac1n\sum_{\mu=1}^n e^C\left(\|\theta_{\mu}(0)\|_{\infty}  +1\right)^3\;.
		\end{align}
		Since we are assuming that the weights are initialized in an i.i.d. fashion with finite third order momenta, we conclude, by the Law of Large Numbers, that the right hand side converges to a constant a.s. In particular, there exists a constant $C>0$, independent of $n$, for which  
		\begin{align}
			\frac1n\sum_{\mu=1}^n \left(\|\theta_{\mu}(t)\|_{\infty}  +1\right)^3
			\leq C
			\;\;\;,\;\; \text{ a.s.}\;
			\;
		\end{align}
		Therefore 
		\begin{align}
			\label{dNTK1}
			\nonumber
			\left| \partial_t \mathrm{NTK}_{zw}^{(n)}(\theta(t)) \right|
			& \leq 
			\frac1n\sum_{\mu=1}^n C C_{\sigma}^3 \left(\|\theta_{\mu}(t)\|_{\infty}  +1\right)^3 
			\frac{1}{\sqrt{n}}\sum_{l=1}^N\left|f_{\theta(t)}(X_l)-Y_l\right|
			\\
			& \leq
			\frac{NC\|E(0)\|_{\infty}}{\sqrt{n}}  e^{-\lambda t}\;,
		\end{align}
		for all $t\in \mathcal T$.  
		
		Finally, we return to~\eqref{dNTK} and use~\eqref{decayError} to show that    
		\begin{align}
			\label{NTKdifFinal}
			|\mathrm{NTK}_{zw}^{(n)}(\theta(t))-\mathrm{NTK}_{zw}^\infty|
			&\leq 
			\frac{C\;(\log n)^{1/2+}}{\sqrt{n}} + 
			\frac{NC\|E(0)\|_{\infty}}{\sqrt{n}\lambda} \;,
		\end{align}
		which can be made smaller than $\epsilon/2$ by an appropriate choice of $n$. This shows that $\mathcal T$ is open and since it is also closed an non-empty we finally conclude that $\mathcal T=[0,+\infty[$. Note that, in particular all derived estimates are, as a consequence, valid for all $t\geq0$.

		\vspace{0,5cm}
		
		We now take the opportunity to study and compare the evolution (during training) of a generic output of both the finite width and infinite width networks.

		Recall that the ``infinite width'' network with kernel  $\mathrm{NTK}^\infty$, associated to the finite width network $f_{\theta(t)}$, is the function
		$f_{\infty}(\,\cdot\,,t):\mathbb{R}\rightarrow \mathbb{R}$ defined as follows: first we define it on the training set $\{X_i\}$ as the solution to the initial value problem 
		\begin{equation}
			\label{def_fInifinite1}
			\left\{
			\begin{array}{l}
				\partial_t f_{\infty}(X_i,t)=-\mathrm{NTK}^{\infty}_{X_iX_j}(f_{\infty}(X_j,t)-Y_j)
				\\
				f_{\infty}(X_i,0)=f_{\theta(0)}(X_i)\;,
			\end{array}
			\right.
		\end{equation}
		with summation over the repeated indices.  For a general $x$ we define  $f_{\infty}(x,t)$ as the solution to 
		\begin{equation}
			\label{def_fInifinite2}
			\left\{
			\begin{array}{l}
				\partial_t f_{\infty}(x,t)=-\mathrm{NTK}^{\infty}_{xX_j}(f_{\infty}(X_j,t)-Y_j)
				\\
				f_{\infty}(x,0)=f_{\theta(0)}(x)\;.
			\end{array}
			\right.
		\end{equation}
		Once again we stress the fact that, in our framework, the infinite kernel $\mathrm{NTK}^{\infty}$ is a.s. time independent; in particular, this means that a.s. the previous systems of ODEs is linear and therefore well posed for all $t\in\mathbb{R}$.   
		%

		To simplify notation let $K^{\infty}_{ij}=\mathrm{NTK}^{\infty}_{X_iX_j}$ and set $y^{\infty}_i=f_{\infty}(X_i,\theta(t))$ and $E^{\infty}_i = y^{\infty}_i-Y_i$. 
		Then the defining system of ODEs~\eqref{def_fInifinite1} reads  
		\begin{equation}
			\label{NTKOde1}
			\partial_t  E^{\infty}_j=\partial_t y^{\infty}_j=-K^{\infty}_{ji}E^{\infty}_i\;,\;\;j=1,\cdots, N\;.
		\end{equation}

		Since $K^{\infty}=(K^{\infty}_{ij})$ is positive definite
		its smallest eigenvalue, which we denote by $\lambda_{\infty}$, is positive and we conclude from~\eqref{NTKOde1} that 
		\begin{equation}
			\label{decayErrorInfty}
			|E^{\infty}_i |=|y^{\infty}_i-Y_i|\leq C\|E^{\infty}(0)\|_{\infty} e^{-\lambda_{\infty} t}\;,
		\end{equation}
		where
		$$\|E^{\infty}(0)\|_{\infty}:=\sup_{i=1,\cdots,N}|E^{\infty}_i(0)|\;.$$

		Let $K^{\infty}_{xi}:=\mathrm{NTK}^{\infty}_{xX_i}$
		and introduce also the notation $y^{\infty}=y^{\infty}(t)=f_{\infty}(x,t)$,
		%
		then
		\begin{equation}
			\label{NTKOde2}
			\partial_t y^{\infty}=-K^{\infty}_{xi}E^{\infty}_i
		\end{equation}
		and since~\eqref{NTKOde1} implies that 
		\begin{equation}
			E^{\infty}_i=- (K^{\infty})^{-1}_{ij}\partial_t E^{\infty}_j\;, 
		\end{equation}
		using~\eqref{NTKOde2} we conclude that
		\begin{align}
			\label{eq:Solution to Infinite Netowrk Along Training}
			y^{\infty}(t) =&y^{\infty}(0)-\int_0^tK^{\infty}_{xi}E^{\infty}_i(s)ds
			\\
			=&y^{\infty}(0)+K^{\infty}_{xi}(K^{\infty})^{-1}_{ij}\int_0^t\partial_tE^{\infty}_j(s)ds
			\\
			=&y^{\infty}(0)+K^{\infty}_{xi}(K^{\infty})^{-1}_{ij}\left(E^{\infty}_j(t)-E^{\infty}_j(0)\right)\;.
		\end{align}

		Taking limits and using~\eqref{decayErrorInfty} we arrive at 
		\begin{equation}
			\lim_{t\rightarrow\infty}y^{\infty}(t)=y^{\infty}(0)-K^{\infty}_{xi}(K^{\infty})^{-1}_{ij}\left(y^{\infty}_j(0)-Y_j\right)\;,
		\end{equation}
		that should be compared with~\cite[Eq. 5]{Arora}. 
		
		%
		%
		
		Now we are able to estimate the difference between the finite and infinite width nets by noting that 

		\begin{eqnarray*}
			\label{NTKOdeDif}
			\partial_t (y_j^{\infty}-y_j) 
			&=& 
			-K^{\infty}_{ji}E^{\infty}_i+K_{ji}E_i
			\\ 
			&=&
			-K^{\infty}_{ji}(y_i^{\infty}-y_i)+(K_{ji}-K^{\infty}_{ji})E_i\;.
		\end{eqnarray*}
		Integrating the previous as a first order linear equation in $d=(d_1,\ldots,d_N)$, $d_i=y_i^{\infty}-y_i$, and recalling that $d(0)=0$ by construction, we get
		%
		%
		%
		$$
		d(t)=e^{-K^{\infty}t}\int_0^t(K(s)-K^{\infty})E(s)e^{K^{\infty}s} ds\;.
		$$
		Using~\eqref{decayError} and~\eqref{NTKdifFinal} we then conclude that 
		\begin{eqnarray}
			|y_i^{\infty}(t)-y_i(t)|\leq \|d(t)\|_{\infty}
			&\leq& 
			e^{-\lambda_{\infty}t}\frac{C \|E(0)\|}{\lambda}  \frac{(\log(n))^{1/2+}}{\sqrt{n}}
			\int_0^t e^{(\lambda_{\infty}-\lambda)s} ds
			\\
			&\leq&
			\frac{C \|E(0)\|}{\lambda}  \frac{(\log(n))^{1/2+}}{\sqrt{n}}e^{-\min\{\lambda_{\infty},\lambda\}t}\;.
		\end{eqnarray}
		%
		%
		%
		%
		
		For general outputs we have
		\begin{eqnarray*}
			\label{NTKOdeDif}
			\partial_t (y^{\infty}-y) 
			&=& 
			-K^{\infty}_{xi}E^{\infty}_i+K_{xi}E_i
			\\ 
			&=&
			-K^{\infty}_{xi}(y_i^{\infty}-y_i)+(K_{xi}-K^{\infty}_{xi})(y_i-Y_i)\;,
		\end{eqnarray*}
		which, since $y^{\infty}(t)=y(0)$, we can integrate to obtain (provided $|x|\leq1$)  
		\begin{eqnarray*}
			|y(t)-y^{\infty}(t)|&\leq & 
			\int_0^t |K^{\infty}_{xi}||y_i(s)-y_i^{\infty}(s)|ds+\int_0^t \left | K_{xi}(s)-K^{\infty}_{xi}\right||y_i(s)-Y_i|ds\\
			&\leq & 
			\left(\max_{i=1,\cdots,N}|K^{\infty}_{xi}|+1\right)
			\frac{C \|E(0)\|}{\min\{\lambda_{\infty},\lambda\}}  \frac{(\log(n))^{1/2+}}{\sqrt{n}}\;.
		\end{eqnarray*}
		This should be compared with~\cite[Theorem 3.2]{Arora}.
		
		%
		%


		\subsection{A slight refinement}
		
		The result just proved in this section shows that for any $\delta >0$, there are: $C>0$ and a set $\cA_\delta \subset \mathbb{R}^{4n}$ of measure $1-\delta$; such that for any $x$ with $|x|\leq 1$ and $\theta(0) \in \cA_\delta$
		\begin{eqnarray*}
			|f_{\theta(t)}(x)-f_{\infty}(x,t)|&\leq & \left(\max_{i=1,\cdots,N}|K^{\infty}_{xi}|+1\right)
			\frac{C \|E(0)\|}{\min\{\lambda_{\infty},\lambda\}}  \frac{(\log(n))^{1/2+}}{\sqrt{n}}\;,
		\end{eqnarray*}
		for any $n \in \mathbb{N}$. However, it is useful in some applications to let $\delta$ depend on $n$ in a manner that $\delta \to 0$ as $n \to +\infty$. Unfortunately, in doing so the quantity $\| E(0) \|$ cannot be bounded uniformly. Indeed, $\| E(0) \| \leq C_\delta + \max_{i=1,\cdots,N}|Y_{i}|$, where $C_\delta>0$ is the constant introduced in equation \ref{probInit}, which in turn implicitly defines $\cA_\delta$. Now, it follows from the Berry-Esseen theorem that
		$$-\frac{1}{\sqrt{n}} + \mathbb{P}( |G| \leq C_\delta )  \lesssim \mathbb{P}( |f_{\theta(t)}(x)| \leq C_\delta ) \lesssim \frac{1}{\sqrt{n}} + \mathbb{P}( |G| \leq C_\delta ) ,$$
		with $G$ a Gaussian. Hence, $\mathbb{P}( |G| \leq C_\delta ) \sim 1-e^{-C_\delta^2}$. Then, equation \ref{probInit} is guaranteed to hold if
		$$-\frac{1}{\sqrt{n}} + (1-e^{-C_\delta^2})  \gtrsim 1-\delta ,$$
		which we can rewrite as
		$$
		\delta \gtrsim e^{-C_\delta^2} + \frac{1}{\sqrt{n}} .
		$$
		We conclude the following.
		
		\begin{proposition}
			There are positive constants $C_1,C_2>0$ with the following significance. Let $\delta_n$ and $C_{\delta_n}$ be a pair of sequences satisfying 
			\begin{equation}\label{eq:delta and C}
				C_1 \delta_n >  e^{-C_{\delta_n}^2} + \frac{1}{\sqrt{n}} .
			\end{equation}
			Then, there is a set $\cA_{\delta_m} \subset \mathbb{R}^{4n}$ of measure $1-\delta_n$; such that for any $x$ with $|x|\leq 1$ and $\theta(0) \in \cA_{\delta_n}$
			\begin{eqnarray*}
				|f_{\theta(t)}(x)-f_{\infty}(x,t)|&\leq & \left(\max_{i=1,\cdots,N}|K^{\infty}_{xi}|+1\right)
				\frac{C_2 \|E(0)\|}{\min\{\lambda_{\infty},\lambda\}}  \frac{(\log(n))^{1/2+}}{\sqrt{n}}\;,
			\end{eqnarray*}
			for any $n \in \mathbb{N}$. Moreover, 
			$$\| E(0) \| \leq C_{\delta_n} + \max_{i=1,\cdots,N}|Y_{i}|.$$
		\end{proposition}
		
		\begin{example}\label{ex:sequences delta and C}
			Consider the case $\delta_n= C \frac{(\log n)^{1/2+}}{\sqrt{n}}$ for a constant $C>0$ and 
			$$C_{\delta_n}^2= \frac{1}{2} \log n - \Big( \frac{1}{2} \Big)^{1+} \log \log n.$$
			We can check that these satisfy \ref{eq:delta and C} and so the conclusion of the previous proposition is valid with
			\begin{eqnarray*}
				|f_{\theta(t)}(x)-f_{\infty}(x,t)|&\lesssim &  \frac{(\log(n))^{1+}}{\sqrt{n}}\;,
			\end{eqnarray*}
		\end{example}

	\end{document}